\documentclass{article}

\usepackage[final]{neurips_2019}

\usepackage[utf8]{inputenc} 
\usepackage[T1]{fontenc}    
\usepackage{hyperref}       
\usepackage{url}            
\usepackage{booktabs}       
\usepackage{amsfonts}       
\usepackage{nicefrac}       
\usepackage{microtype}  
\usepackage{graphicx}
\usepackage{caption}
\usepackage{mathtools}
\usepackage{thmtools,thm-restate}
\usepackage{subcaption}
\usepackage{multirow}
\usepackage{nicefrac}
\usepackage{algorithm}      
\usepackage{amsthm}
\usepackage[noend]{algpseudocode} 

\newtheorem{theorem}{Theorem}
\newtheorem{lemma}{Lemma}

\newtheorem{definition}{Definition}
\title{Differentially Private Survival Function Estimation}

\author{%
  Lovedeep Gondara \\
  Department of Computing Science\\
  Simon Fraser University\\
  Canada \\
  \texttt{lgondara@sfu.ca} \\
   \And
   Ke Wang \\
   Department of Computing Science\\
   Simon Fraser University\\
   Canada \\
   \texttt{wangk@cs.sfu.ca} \\
}

\begin{document}

\maketitle

\begin{abstract}
Survival function estimation is used in many disciplines, but it is most common in medical analytics in the form of the Kaplan-Meier estimator. Sensitive data (patient records) is used in the estimation without any explicit control on the information leakage, which is a significant privacy concern. We propose a first differentially private estimator of the survival function and show that it can be easily extended to provide differentially private confidence intervals and test statistics without spending any extra privacy budget. We further provide extensions for differentially private estimation of the competing risk cumulative incidence function, Nelson-Aalen's estimator for the hazard function, etc. Using eleven real-life clinical datasets, we provide empirical evidence that our proposed method provides good utility while simultaneously providing strong privacy guarantees.
\end{abstract}

\section{Introduction}
A patient progresses from HIV infection to AIDS after 4.5 years. A study using the patient's data publishes the survival function estimates (a standard practice in clinical research). An adversary, with only access to the published estimates (even in the form of survival function plots), can reconstruct user-level data \cite{wei2018reconstructing,fredrikson2014privacy}. Effectively leading to the disclosure of sensitive information. This is just one scenario. The survival function is used for modeling any time to an event, taking into account that some subjects will not experience the event at the time of data collection. The survival function is used in many domains, some examples are the duration of unemployment (in economics); time until the failure of a machine part (in engineering); time to disease recurrence, time to infection, time to death (in healthcare); etc. 

Our personal healthcare information is the most sensitive private attribute, protected by law, violations of which carry severe penalties. And as the initial example suggests, of all application areas, information leakage in the healthcare domain is the most serious issue and is our focus in this study. For estimation of the survival function, we focus on the Kaplan-Meier's (KM) \cite{kaplan1958nonparametric} non-parametric method. KM's method is ubiquitous in clinical research. A quick search of the term on PubMed\footnote{A free search engine indexing manuscripts and abstracts for life sciences and other biomedical topics. Link - \url{https://www.ncbi.nlm.nih.gov/pubmed/}} yields more than 110,000 results. It is not an overstatement to say that almost every clinical study uses KM's method to report summary statistics on their cohort's survival. Statistical agencies around the world use this method to report on the survival of the general population or specific disease-related survival estimates.

To best of our knowledge, there does not exist any method that can provide formal privacy guarantees for estimation of survival function using the KM method. The only related work is by Nguy\^{e}n \& Hui \cite{nguyen2017differentially}, which uses the output and objective perturbation for regression modeling of discrete time to event data. The approach is limited to ``multivariate" regression models and cannot be directly used to estimate survival function in a differentially private fashion. One can argue that generative models such as the differentially private generative adversarial networks \cite{xie2018differentially,zhang2018differentially,triastcyn2018generating,beaulieu2017privacy,esteban2017real,yoon2018pategan} can be trained to generate differentially private synthetic data. Which can then be used to estimate the survival function. But, GANs do not generalize well to the datasets typically encountered for our use-case (very small sample size (can be less than a hundred), highly constrained dimensionality ($d \in [2,3]$), a mixture of categorical and continuous variables, no data pre-processing (scaling, etc.) allowed, etc.). 

We propose the first differentially private method for estimating the survival function based on the KM method. Grounded by the core principles of differential privacy, our method guarantees the differentially private estimation of the survival function. Also, we show that our method easily extends to provide differentially private confidence intervals and differentially private test statistics (for comparison of survival function between multiple groups) without any extra privacy cost. We further extend our method for differentially private estimation of the competing risk cumulative incidence function and the hazard function using the Nelson-Aalen estimator \cite{nelson1972theory,nelson1969hazard,aalen1978nonparametric} (other popular estimates in clinical research). Using eleven real-life clinical datasets, we provide empirical evidence that our proposed method provides good utility while simultaneously providing strong privacy guarantees. Lastly, we release our method as an R\footnote{Most often used programming language in medical statistics} \cite{Rcore} package for rapid accessibility and adoption.

\section{Preliminaries and Technical Background}
We use this section to introduce the concepts central to the understanding of our method.
\subsection{Survival Function}\label{sec:prelim_survival}
The survival function is used to model time to event data, where the event may not have yet occurred (but the probability of occurrence is non-zero). Such as for HIV infection to AIDS timeline data, at the end of the follow-up period, some patients would have progressed (our event of interest), while others would not have yet progressed (censored observations). Accounting for censored observations (patients that never experience the event during our follow-up) is the central component in the estimation of the survival function. Formally, 
\begin{equation}
    S(t) = P(T > t) = \int_{t}^{\infty} f(u) \ du = 1 - F(t) 
\end{equation}
this gives the probability of not having an event just before time $t$, or more generally, the probability that the event of interest has not occurred by time $t$. 

In practice, survival function can be estimated using more than one approach. Several parametric methods (that make assumptions on the distribution of survival times) such as the ones based on the exponential, Weibull, Gompertz, and log-normal distributions are available. Or one can opt for the most famous and most often used non-parametric method (Kaplan-Meier's method \cite{kaplan1958nonparametric}), which does not assume how the probability of an event changes over time. Our focus in this paper is the latter, which has become synonymous with survival models in clinical literature.  KM estimator of the survival function is defined as follows
\begin{equation}\label{eqn:surv}
    \hat{S}(t) = \prod_{j:t_j \le t} \frac{r_j - d_j}{r_j}
\end{equation}
where $t_j, (j \in [1,\cdots,k])$ is the set of $k$ distinct failure times (not censored), $d_j$ is the number of failures at $t_j$, and $r_j$ are the number of individuals ``at risk" before the $j$-th failure time. We can see that the function $\hat{S}(t)$ only changes at each failure time, not for censored observations, resulting in a ``step" function (the characteristic feature of KM estimate).

\subsection{Differential Privacy}
Releasing any form of data (raw data, function estimates, derived statistics, etc.) can potentially leak sensitive information \cite{dinur2003revealing}. Differential privacy \cite{Dwork:2006:CNS:2180286.2180305}, a \textit{de facto} standard for providing provable privacy guarantees provides us with the method to quantify such information leakage. Differential privacy is based on the concept of \textit{neighbouring} datasets, that is
\begin{definition}
\emph{(Neighbouring datasets \cite{Dwork:2006:CNS:2180286.2180305})} Two datasets $D,D'$ are said to be neighbouring if 
\begin{equation}
    \exists \ i \in D \ \text{s.t.} \ D \char`\\  {i}  = D'
\end{equation}
\end{definition}
which means that $D$ and $D'$ are neighboring datasets if they only differ in any one row (one user). This notion is central to individual-level privacy guarantees, we will recall this definition when we provide the proof for privacy guarantees of our proposed method.

\begin{definition}
\emph{(Differential privacy \cite{Dwork:2006:CNS:2180286.2180305})} A randomized mechanism $\mathcal{M}: D^n \rightarrow \mathbb{R}^d$ preserves $(\epsilon,\delta)$-differentially privacy if for any pair of databases ($D,D' \in D^n$) such that $d(D,D')=1$, and for all sets $S$ of possible outputs:

\begin{equation}
    \text{Pr}[\mathcal{M}(D) \in S] \le e^{\epsilon} \text{Pr}[\mathcal{M}(D') \in S] + \delta 
\end{equation}
\end{definition}

The definition guarantees that it is information-theoretically impossible for an adversary to infer whether the input dataset to the mechanism $\mathcal{M}$ is $D$ or $D'$ (where $D,D'$ are neighboring datasets, that is, $d(D,D')=1$) beyond a certain probability. The probability is a multiplicative factor of $e^\epsilon$. We can see that by making $\epsilon$ smaller, we can make the probability small, leading to a strong degree of \textit{plausible deniability} for an individual's presence or absence in the dataset. The definition above allows the ``relaxation" of strict privacy guarantees by an additive factor of $\delta$. Simply put, we allow the privacy to be broken with an additive probability $\delta$. As is clear from the definition, smaller ($\epsilon,\delta$) provide stronger privacy guarantees. When $\delta=0$, we have pure-$\epsilon$ differential privacy. 

Differential privacy has many interesting properties, here we briefly introduce the most useful for our use-case. That is the \textit{post-processing}. The post-processing theorem states that differential privacy is immune to post-processing. That is, any function acting solely on the output of a differentially private mechanism is also differentially private, formally

\begin{theorem}
\emph{(Post processing \cite{Dwork:2006:CNS:2180286.2180305})} Let $\mathcal{M}: D^n \rightarrow \mathbb{R}$ be a randomized mechanism that is ($\epsilon, \delta$)-differentially private. Let $\mathbb{R} \rightarrow \mathbb{R}'$ be a deterministic function. Then $f \circ M: D^n \rightarrow \mathbb{R}'$ is ($\epsilon, \delta$)-differentially private.
\end{theorem}

This result is directly used in our proposed model. Where after adding noise to our main quantity of interest, we claim that any estimates derived solely from the differentially private quantity are differentially private.

\section{Differentially Private Estimation of Survival Function}
Now we introduce our method for differentially private estimation of the survival function using the Kaplan-Meier's method. We follow the basic principles of differential privacy to ensure that our estimate of the survival function is differentially private. We subsequently show that following our simple approach, it is possible to estimate a wide variety of accompanying statistics (such as the confidence intervals, comparison test statistics, etc.) in a differentially private way without spending any extra privacy budget.

\subsection{Estimation}
Before we begin, we recap some of the notations introduced in Section \ref{sec:prelim_survival}. We have a vector of unique failure time points ($t_j, j \in [1, \cdots, k]$), and for each time point, we have a corresponding number of subjects at risk $r_j$ (number of subjects not experiencing a progression/event up to that time point), and we have the number of subjects experiencing the event at that time point (number of progressions/events), which we denote as $d_j$.

To start, we create a \emph{partial} matrix, $M$, where for the first time point $t_1$, we have the data on the number of events ($d_1$) and the number at risk ($r_1$), and for the rest of the time points ($t_j, j \in [2, \cdots, k]$), we only have the data on the number of events ($d_j$) (we use this initial setup to ensure our privacy guarantees hold, as we explain in following paragraph and in Section \ref{sec:privacy_proof}). Then, using the derived $L_1$ sensitivity ($\mathcal{S}$) of $M$ (details in Section \ref{sec:privacy_proof}), we draw a noise matrix $Z$ from the Laplace distribution ($Lap(\nicefrac{\mathcal{S}}{\epsilon})$), where $\epsilon$ is the privacy parameter and $Z$ is of the same size as $M$. Adding $Z$ to $M$ ($M' = M + Z$) guarantees that $M'$ is differentially private (formal proof in Section \ref{sec:privacy_proof}).

Please note that the matrix $M'$, although differentially private, is still incomplete as we only have the number at-risk for the first time point ($r_1'$ for $t_1$, after noise addition). To complete the matrix, we derive the rest of the ``at-risk" population using our noisy events ($d'_j$). That is, to obtain the number at risk ($r'_j$) for a subsequent time point ($t_j$), we subtract the number of events ($d'_{j-1}$) from the number at risk ($r'_{j-1}$) for the previous timepoint. This approach is similar to how number at-risk is generally calculated in a non-noisy case, as cases that have experienced an event are no longer at risk for the same event and are removed from the risk set, we use the noisy number of events ($d'_j$) to ensure our privacy guarantees hold. We present our method succinctly as Algorithm \ref{main_algo} followed by a detailed discussion.

\begin{algorithm}
\caption{Differentially Private Estimation of $\hat{S}(t)$}\label{main_algo}
\begin{algorithmic}[1]
\Procedure{DP}{$\hat{S}(t)$}
\State Create a partial matrix $M; [r_1,d_j] \in M$; for every $t_j$
\State $M' = M + Lap(\nicefrac{\mathcal{S}}{\epsilon}); [r_1', d_j'] \in M'$
\For{$j=2,\cdots,k$}
\State {$r_j' = r_{j-1}' - d_{j-1}'$}
\EndFor
\State $\hat{S}'(t) = \prod_{j:t_j \le t} \frac{r_j' - d_j'}{r_j'}$
\State \Return $\hat{S}'(t)$
\EndProcedure
\end{algorithmic}
\end{algorithm}

\subsubsection{Discussion} 
We use this paragraph to briefly discuss Algorithm \ref{main_algo}. We begin with the noticeable simplicity of the procedure, that is, the minimal changes required to the original estimation procedure to make it differentially private. This simplistic approach serves a crucial two-fold role. This boosts the accessibility of our differentially private version (it can be implemented using any readily available software package), and aids in ensuring that many other required and reported statistics with the survival function (test statistics, confidence intervals, etc.) are differentially private without spending any extra privacy budget (details follow). Also, in our method, the required changes for differential privacy come with no computational overhead compared to the original estimation (our method is equally computationally cheap). 

An important observation is that with current Algorithm \ref{main_algo}, using $M'$ for estimating the survival function, we might have scenarios where $d'_j$ or rarely $r'_j$ are negative, leading to the \emph{non-monotonic} behavior of the differentially private survival function. We fix this issue as follows: After completion of $M'$, we check to ensure that any noisy values are not violating our data integrity constraints (i.e. $r_j', d_j' < 0 $), if they are, we replace such values by 0\footnote{Once the at-risk population is 0, we do not consider any future time points.}. This extra step does not require any additional privacy budget and it does not violate our privacy claims, as it is a standard case of \emph{post-processing} in differential privacy, in spirit similar to label smoothing \cite{wang2016using} or enforcing data integrity constraints \cite{flaxman2019empirical}. Next, we provide the formal privacy guarantees and further details on how our proposed method can be easily extended for differentially private estimation of ``other" associated statistics. 

\subsection{Privacy Guarantees}\label{sec:privacy_proof}
Now we are ready to formally state the differential privacy guarantees of our proposed method. Before we state our main theorem, we start with a supporting Lemma for establishing the global $L_1$ sensitivity ($\mathcal{S}$) of our method.
\begin{lemma}
$L_1$ sensitivity ($\mathcal{S}$) of $M$ is two.
\end{lemma}
\begin{proof}
As initial $M$ only contains count variables for the number of events (for all time points) and number at risk (only for the first time point). Adding or removing any single individual can change the counts by at most two (that is being in at-risk group and having an event).
\end{proof}
\begin{theorem}\label{thm_main}
Algorithm \ref{main_algo} is $\epsilon$-differentially private.
\end{theorem}
\begin{proof}
Let $M \in \mathbb{R}^d$ and $M^* \in \mathbb{R}^d$, such that the $L_1$ sensitivity, $\mathcal{S}$, is $||M - M^*||_1 \le 1$, and let $f(.)$ denote some function, $f : \mathbb{R}^d \rightarrow \mathbb{R}^k$. Let $p_M$ denote the probability density function of $\mathcal{Z}(M,f,\epsilon)$, and let $p_{M^*}$ denote the probability density function of $\mathcal{Z}(M^*,f,\epsilon)$, we compare both at some arbitrary point $q \in \mathbb{R}^k$.

\begin{equation}
\begin{split}
        & \frac{p_M(q)}{p_{M^*}(q)} = \prod_{i=1}^k \bigg ( \frac{\exp(-\frac{\epsilon|f(M)_i - q_i|}{\Delta f})}{\exp(-\frac{\epsilon|f(M^*)_i - q_i|}{\Delta f})} \bigg ) \\
        & =  \prod_{i=1}^k \exp \bigg ( \frac{\epsilon(|f(M^*)_i - q_i| - |f(M)_i - q_i|)}{\Delta f} \bigg ) \\
        & \le  \prod_{i=1}^k  \exp \bigg ( \frac{\epsilon|f(M)_i - f(M^*)_i|}{\Delta f} \bigg ) \\
        & = \exp \bigg (\frac{\epsilon||f(M) - f(M^*)||_1}{\Delta f} \bigg ) \\
        & \le \exp(\epsilon)
\end{split}
\end{equation}
last inequality follows from the definition of sensitivity $\mathcal{S}$

As our function estimation ($\hat{S}'(t)$) uses everything from $M'$ (our differentially private version of $M$) and nothing else from the sensitive data, our survival function estimation is differentially private by the post-processing Theorem \cite{Dwork:2014:AFD:2693052.2693053}.
\end{proof}

\section{Extending to Other Estimates}
As mentioned in the introduction and the previous section, one of the advantages of our approach is its easy extension to other essential statistics often required and reported along with the estimates of the survival function. Such as the confidence intervals, test statistics for comparing the survival function distributions among patient groups, etc. Here we formally define the extensions with their privacy guarantees.

\subsection{Confidence Intervals and Test Statistics}
When reporting survival function estimates, it is often required to include the related confidence intervals, reported to reflect the uncertainty of the estimate. And for the group comparison, such as comparing the infection rates between two treatment arms of a clinical trial, hypothesis testing is used with the help of test statistic. So, it is of paramount interest to provide the differentially private counterparts of both (confidence intervals and test statistics). We start with the confidence intervals.

\subsubsection{Confidence Intervals} 
Confidence intervals for survival function estimates are calculated in a ``point-wise" fashion, that is, they are calculated at discrete time-points whenever an event is observed (for the same time points at which the survival function changes its value). We start with proving that the calculations required for obtaining confidence intervals are differentially private following the changes made to the data in Algorithm \ref{main_algo}.

\begin{theorem}
Confidence Intervals for $\hat{S}'(t)$ are $\epsilon$-differentially private.
\end{theorem}
\begin{proof}
There are more than one type of confidence intervals available for the survival function. Here we focus on the most often used Greenwood's \cite{greenwood1926report} linear-point-wise confidence intervals.

Greenwood's formula for the confidence intervals is given as
\begin{equation}
    \hat{S}(t) \pm z_{1-\nicefrac{\alpha}{2}} \sigma_S(t)
\end{equation}
where
\begin{equation}
    \sigma_s^2(t) = \hat{V}[\hat{S}(t)]
\end{equation}
and
\begin{equation}
    \hat{V}[\hat{S}(t)] = \hat{S}(t)^2 \sum_{t_j \le t} \frac{d_j}{r_j(r_j - d_j)}
\end{equation}
Replacing by their respective differentially private counterparts from Algorithm \ref{main_algo}.
\begin{equation}
    \hat{V}'[\hat{S}(t)] = \hat{S}'(t)^2 \sum_{t_j \le t} \frac{d_j'}{r_j'(r_j' - d_j')}
\end{equation}
estimate for $\hat{V}'[\hat{S}(t)]$ is now differentially private, using it in conjunction with $\hat{S}'(t)$ makes the confidence intervals differentially private by the post-processing theorem \cite{Dwork:2006:CNS:2180286.2180305}.
\end{proof}
As we don't need any additional access to the sensitive data for calculating confidence intervals. Hence, calculating and providing differentially private confidence intervals with the differentially private survival function estimates does not incur any additional privacy cost. In other words, we get the differentially private confidence intervals for free.

\subsubsection{Hypothesis Tests} 
Hypothesis tests are often used to compare the distribution of survival function estimates between groups. For example: To compare infection rates between two treatment arms of a clinical trial. Most often used statistical test in such scenarios is the Logrank test \cite{mantel1966evaluation}. Below we show that using our method (Algorithm \ref{thm_main}), the hypothesis testing using the Logrank test is differentially private.
\begin{theorem}
Hypothesis test for $\hat{S}'(t)$ is $\epsilon$-differentially private.
\end{theorem}
\begin{proof}
Logrank test statistic ($Z$) is given as
\begin{equation}
    Z = \frac{\sum_{j=1}^k (O_{1j} - E_{1j})}{\sqrt{\sum_{j=1}^k V_j}}
\end{equation}
where $O_{1j}$ are observed number of failures (events) ($d_{1j}$) and $E_{1j}$ are the expected number of failures at time $j$ in group $1$, we have 
\begin{equation}
    E_{1j} = d_{j}\frac{r_{1j}}{r_{j}}
\end{equation}
and $V_j$ is the variance, given as
\begin{equation}
    V_j = \frac{r_{1j}r_{2j}d_j(r_j - d_j)}{r_j^2 (r_j - 1)}
\end{equation}
Replacing the corresponding quantities by their differentially private counterparts using Algorithm \ref{main_algo}, we get
\begin{equation}
     V_j' = \frac{r'_{1j}r'_{2j}d'_{j}(r'_{j} - d'_{j})}{r'^{2}_{j} (r'_{j} - 1)}
\end{equation}
which makes $V_j'$ differentially private as no other sensitive information is required for its estimation.

Using it in conjunction with $O_{1j}$ and $E_{1j}$, which can be made differentially private following the same argument, makes the test statistic $Z$ differentially private by the post-processing theorem \cite{Dwork:2006:CNS:2180286.2180305}.
\end{proof}

The calculation, again being the case of standard post-processing on differentially private data does not add to our overall privacy budget. Hence, after using Algorithm \ref{main_algo}, we can output the related confidence intervals and the test statistic without spending any additional privacy budget.

\subsection{Competing Risks Cumulative Incidence}
In certain scenarios, we can have more than one type of event. Using our prior example of HIV infection, we might have a scenario where patients die before progression to AIDS, making the observation of progression impossible. Such events (death) that preclude any possibility of our event of interest (progression) are known as competing events. Competing events are a frequent occurrence in clinical data and require specialized estimates that take this phenomenon into account, without which our estimates will be biased. One such estimate is the competing risk cumulative incidence, which is also the most widely used and reported estimate in the literature, akin to the KM estimate, but for competing events. 

Here we show that using Algorithm \ref{main_algo}, we can easily extend differential privacy to the competing risk scenarios.
\begin{theorem}
Competing risk cumulative incidence using our method is $\epsilon$-differentially private.
\end{theorem}
\begin{proof}
Cumulative incidence extends Kaplan-Meier estimator and is given by
\begin{equation}
    \hat{I}_k(t) = \sum_{j:t_j < t} \hat{S}(t_j)\frac{d_{jk}}{n_j} 
\end{equation}
where $d_{jk}$ is the number of events of type $k$ at time $t_{(j)}$ and $\hat{S}(t_j)$ is the standard Kaplan-Meier estimator to time $t_{(j)}$.

Replacing associated quantities with their differentially private counterparts (using same reasoning as Algorithm \ref{main_algo}).
\begin{equation}
    \hat{I}_k(t)' = \sum_{j:t_j < t} \hat{S}(t_j)'\frac{d_{jk}'}{n_j'} 
\end{equation}
Its not hard to see that $\hat{I}_k(t)'$ is differentially private by the post-processing theorem.
\end{proof}

Further statistics associated with the cumulative incidence such as the confidence intervals and hypothesis tests, etc. that directly depend on the quantities made differentially private using Algorithm \ref{main_algo} can be similarly argued to be differentially private.

\subsection{Nelson-Aalen's Estimate of the Hazard Function}
Analogous to the KM estimator of the survival function, another important non-parametric estimator that is often used is the Nelson-Aalen estimator \cite{nelson1969hazard,aalen1978nonparametric} of the cumulative hazard. Nelson-Aalen estimator estimates the hazard at each time point and has a nice interpretation as the expected number of deaths in $(0, t_j]$ per unit at risk.  Although hazard function can be derived using it's relationship with the survival function ($\hat{A}_t = -\log(\hat{S}_t)$), for which we can directly argue differential privacy using our estimate of $\hat{S}'(t)$, there are certain scenarios where we explicitly need to use the Nelson-Aalen estimator or require it for deriving ``other" estimates, such as the Flemming-Harrington's estimate of the survival function, etc. Hence, below we prove that using Algorithm \ref{main_algo}, we can easily guarantee differential privacy of Nelson-Aalen's estimator.

\begin{theorem}
Following Algorithm \ref{main_algo}, Nelson-Aalen estimator of the hazard function is differentially private.
\end{theorem}
\begin{proof}
Nelson-Aalen's estimator is given as
\begin{equation}
    \hat{A}_t = \sum_{t_j \le t} \frac{d_j}{r_j}
\end{equation}
replacing $d_j$ and $r_j$ with their noisy counterparts from Algorithm \ref{main_algo}
\begin{equation}
    \hat{A}_t' = \sum_{t_j \le t} \frac{d_j'}{r_j'}
\end{equation}
where $\hat{A}_t'$ is now differentially private by the post-processing property of differential privacy.
\end{proof}

Similar to the survival function, ``other" statistics associated with the Nelson-Aalen estimator can be argued to be differentially private following Algorithm \ref{main_algo}.

\section{Empirical Evaluation}
Here we present the empirical evaluation of our method on eleven real-life clinical datasets (nine for evaluating KM and two for competing risk) of varying properties. We start with the dataset description for our main comparison.
\subsection{Datasets}\label{sec:datasets}
Nine real-life clinical datasets with time to event information are used to evaluate our proposed method for the KM estimate\footnote{We use two additional datasets for competing risk evaluation in Section \ref{sec:cr}}. Dataset summary is provided in Table \ref{tab:datasets} followed by further dataset-specific details (dataset properties, pre-processing, group comparison details for hypothesis tests, etc.).
\begin{table}[h]
\caption{Datasets used for evaluation of our proposed method, observations are the number of observations (rows) in the dataset. Wide variety of datasets are used to simulate real-world clinical scenarios.}
\label{tab:datasets}
\centering
\begin{tabular}{ll}
\hline
Dataset  & Observations \\ \hline
Cancer   & 228          \\
Gehan    & 42           \\
Kidney   & 76           \\
Leukemia & 23           \\
Mgus     & 1384         \\
Myeloid  & 646          \\
Ovarian  & 26           \\
Stanford & 184          \\
Veteran  & 137         
\end{tabular}
\end{table}

\begin{enumerate}
    \item Cancer: It pertains to the data on survival in patients with advanced lung cancer from the North Central Cancer Treatment Group \cite{loprinzi1994prospective}. Survival time in days is converted into months. Groups compared are survival amongst males and females.
    \item Gehan: This is the dataset from a trial of 42 leukemia patients \cite{cox2018analysis}. Groups compared are the control and treatment groups. 
    \item Kidney: This dataset is for the recurrence times to infection, at the point of insertion of the catheter, for kidney patients using portable dialysis equipment \cite{mcgilchrist1991regression}. Time is converted into months and groups compared are males and females. 
    \item Leukemia: The dataset pertains to survival in patients with Acute Myelogenous Leukemia \cite{miller2011survival}. Time is converted into months and groups compared are the patients receiving maintenance chemotherapy vs no maintenance chemotherapy.
    \item Mgus: This dataset is about natural history of subjects with monoclonal gammopathy of undetermined significance (MGUS) \cite{kyle1993benign}. Time is converted into months and groups compared are males and females.
    \item Myeloid: Dataset is based on a trial in acute myeloid leukemia. Time is converted into months and groups compared are the two treatment arms.
    \item Ovarian: This dataset pertains to survival in a randomized trial comparing two treatments for ovarian cancer \cite{edmonson1979different}. Time is converted into months and groups compared are the different treatment groups.
    \item Stanford: This dataset is the Stanford Heart Transplant data \cite{escobar1992assessing}. Time is converted into months and groups compared are the age groups (above and below median).
    \item Veteran: This dataset has information from randomized trial of two treatment regimens for lung cancer \cite{kalbfleisch2011statistical}. Time is converted into months and groups compared are the treatment arms.
\end{enumerate}

\begin{figure*}[h!]
\centering
\begin{subfigure}{.3\textwidth}
  \includegraphics[width=.9\linewidth]{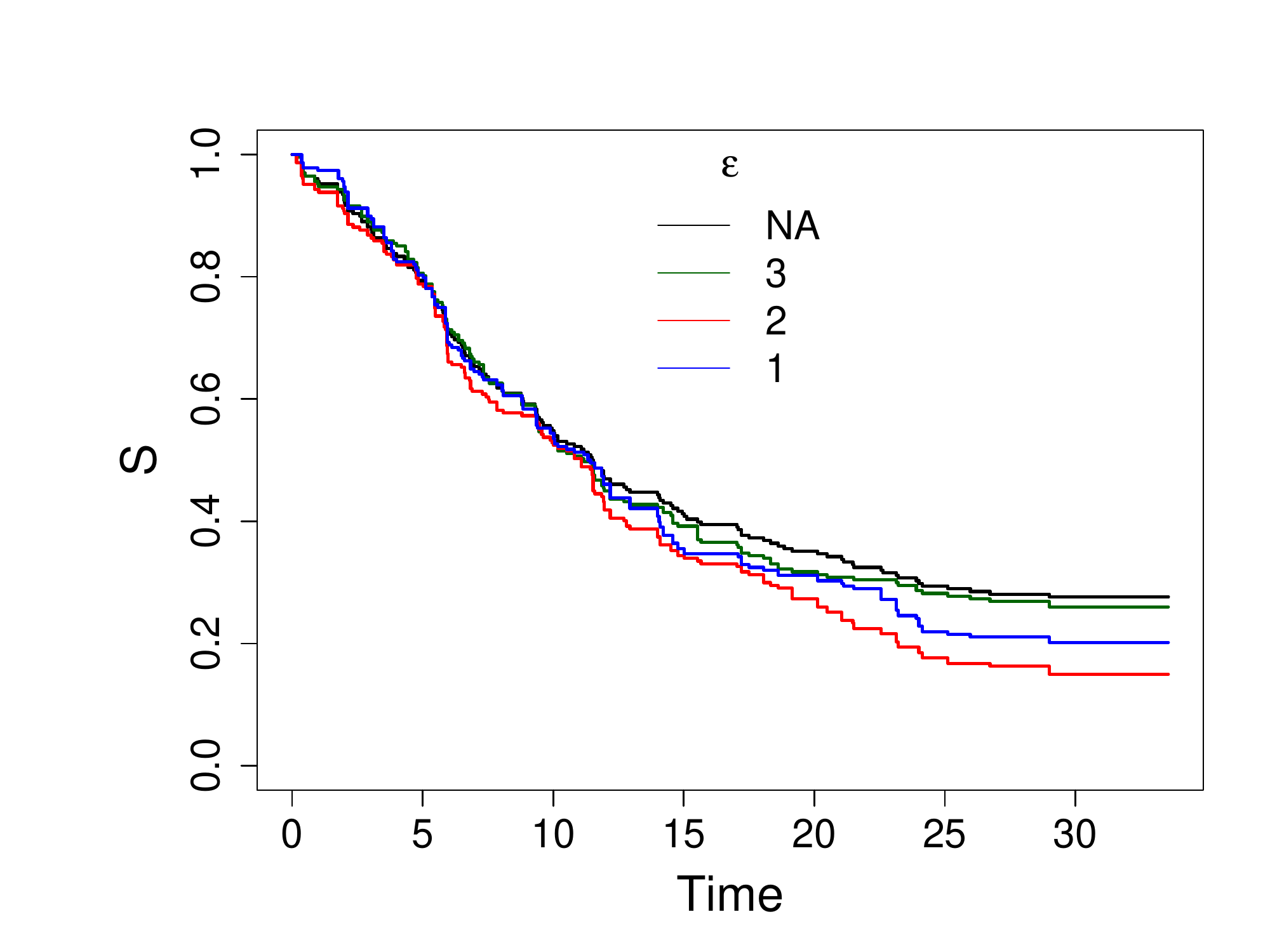}
  \caption{Cancer}
  \label{fig:sub1}
\end{subfigure}
\begin{subfigure}{.3\textwidth}
  \includegraphics[width=.9\linewidth]{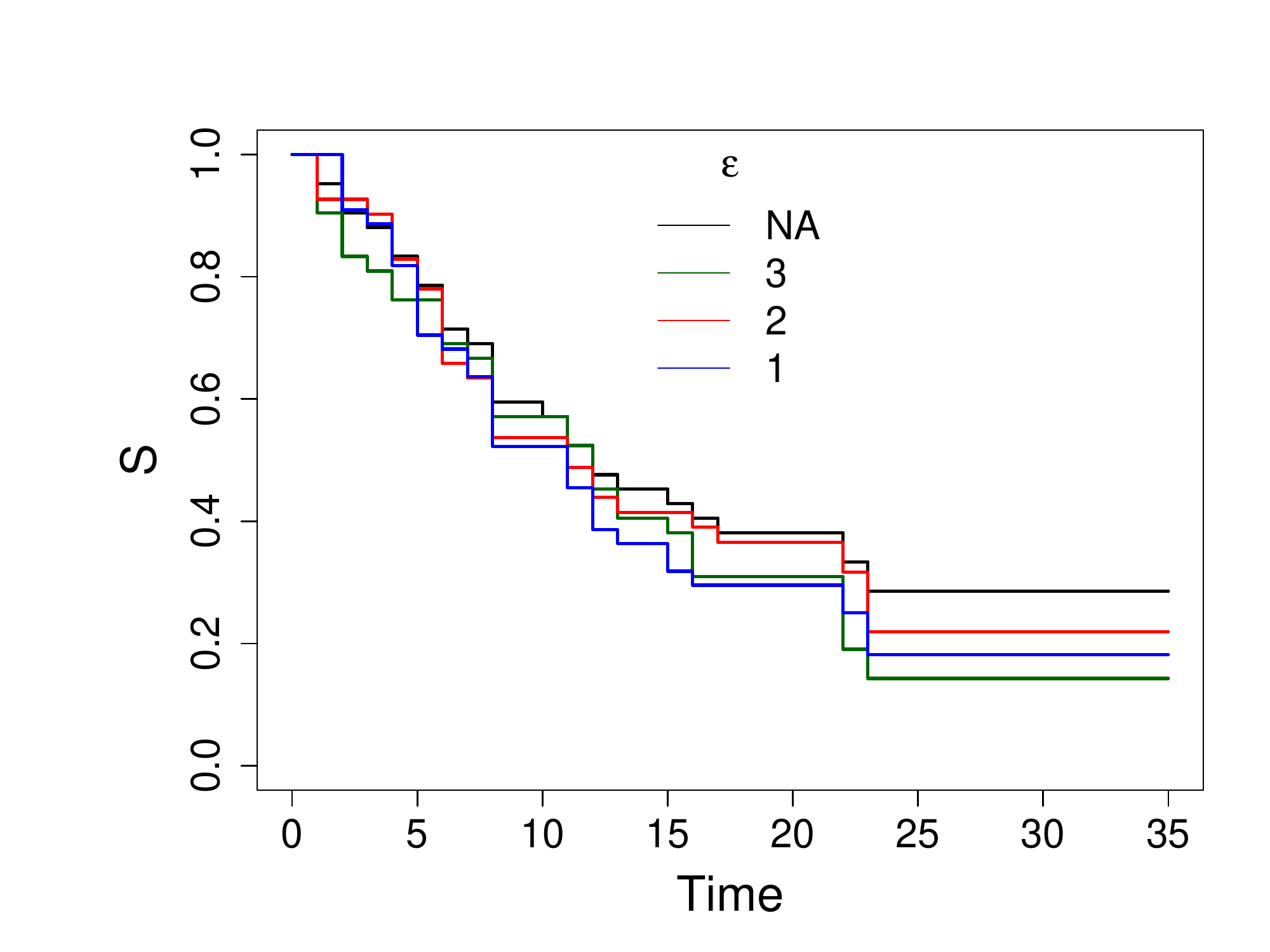}
  \caption{Gehan}
  \label{fig:sub2}
\end{subfigure}
\begin{subfigure}{.3\textwidth}
  \includegraphics[width=.9\linewidth]{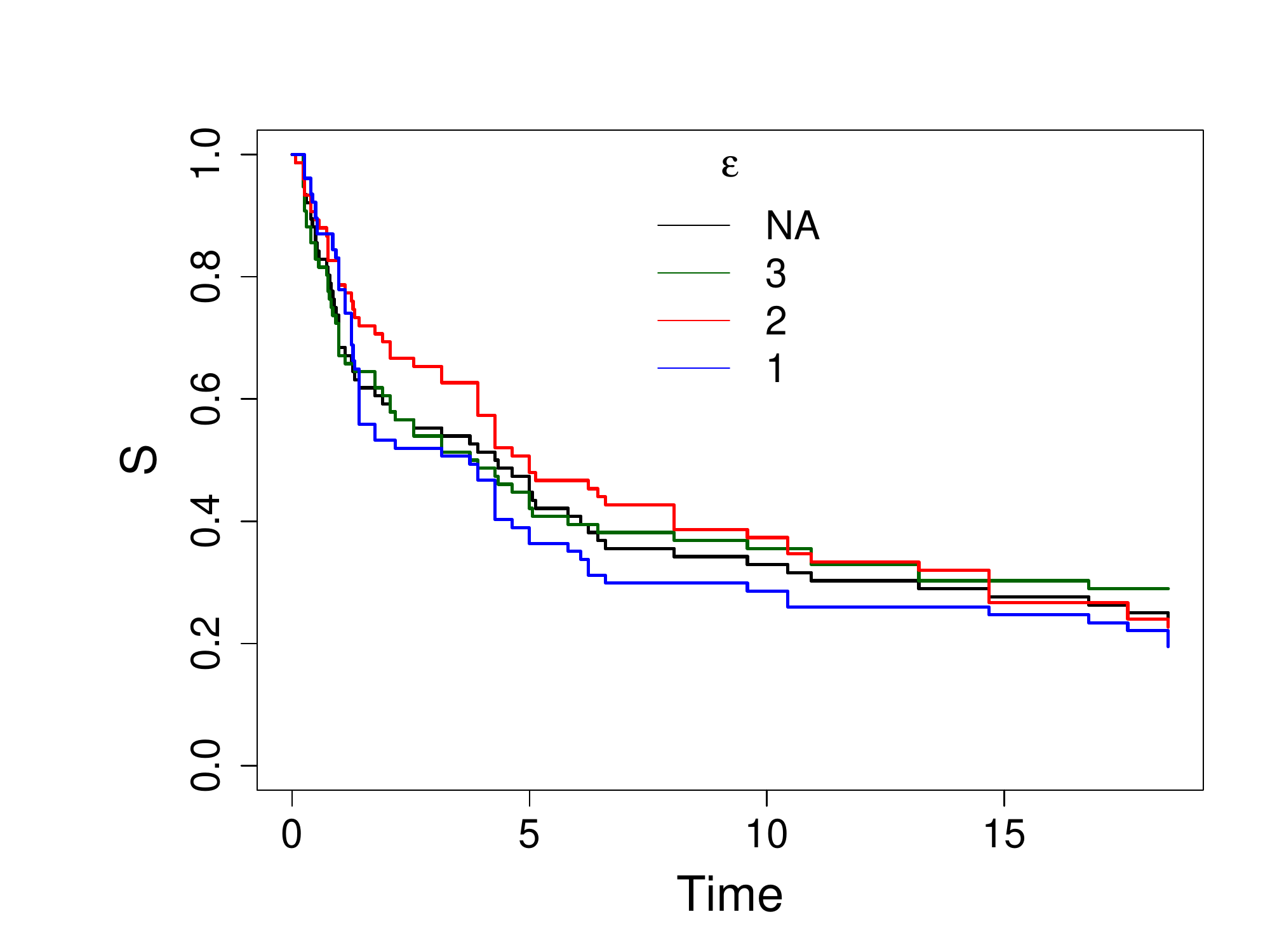}
  \caption{Kidney}
  \label{fig:sub2}
\end{subfigure}

\begin{subfigure}{.3\textwidth}
  \includegraphics[width=.9\linewidth]{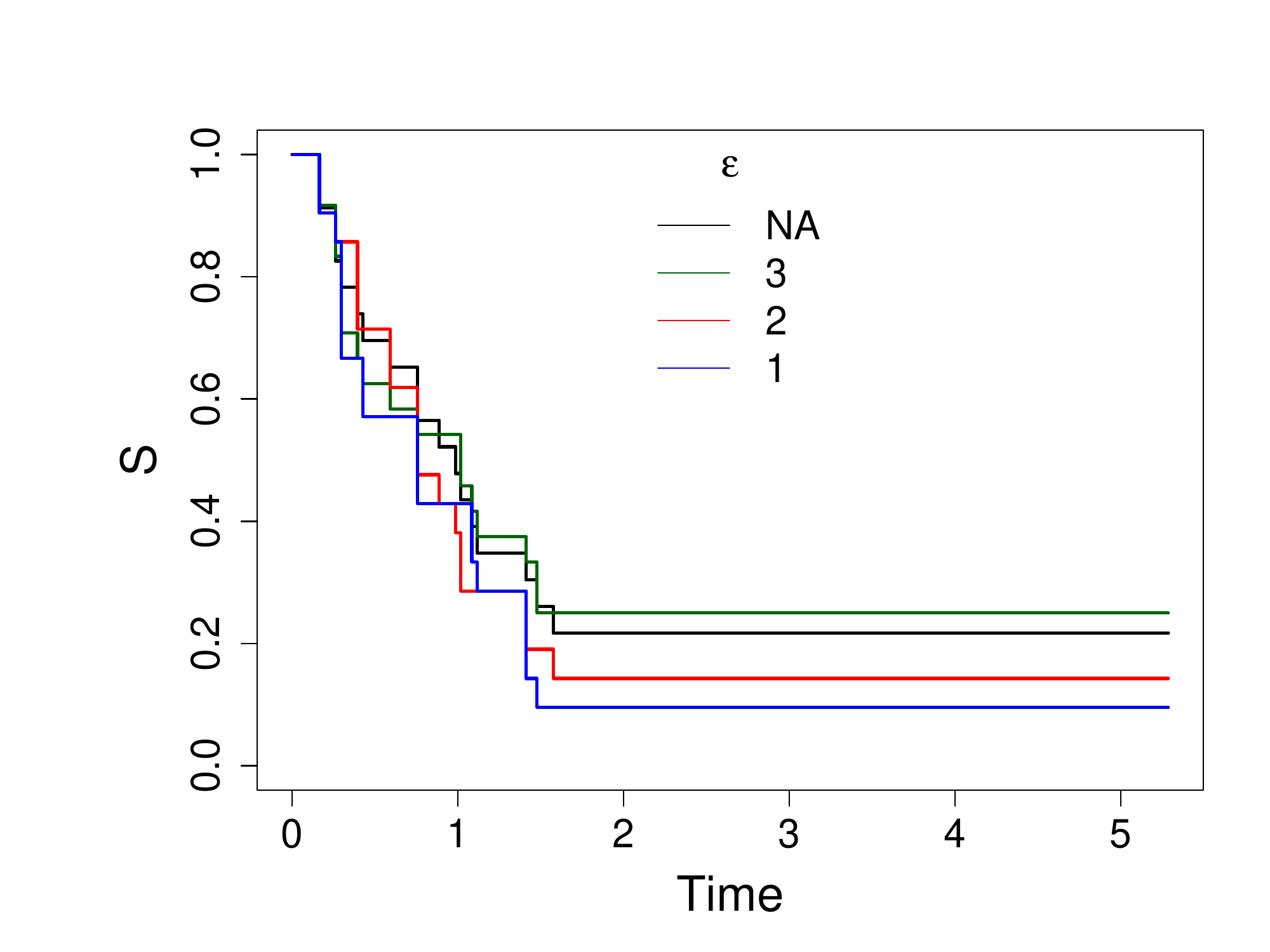}
  \caption{Leukemia}
  \label{fig:sub2}
\end{subfigure}
\begin{subfigure}{.3\textwidth}
  \includegraphics[width=.9\linewidth]{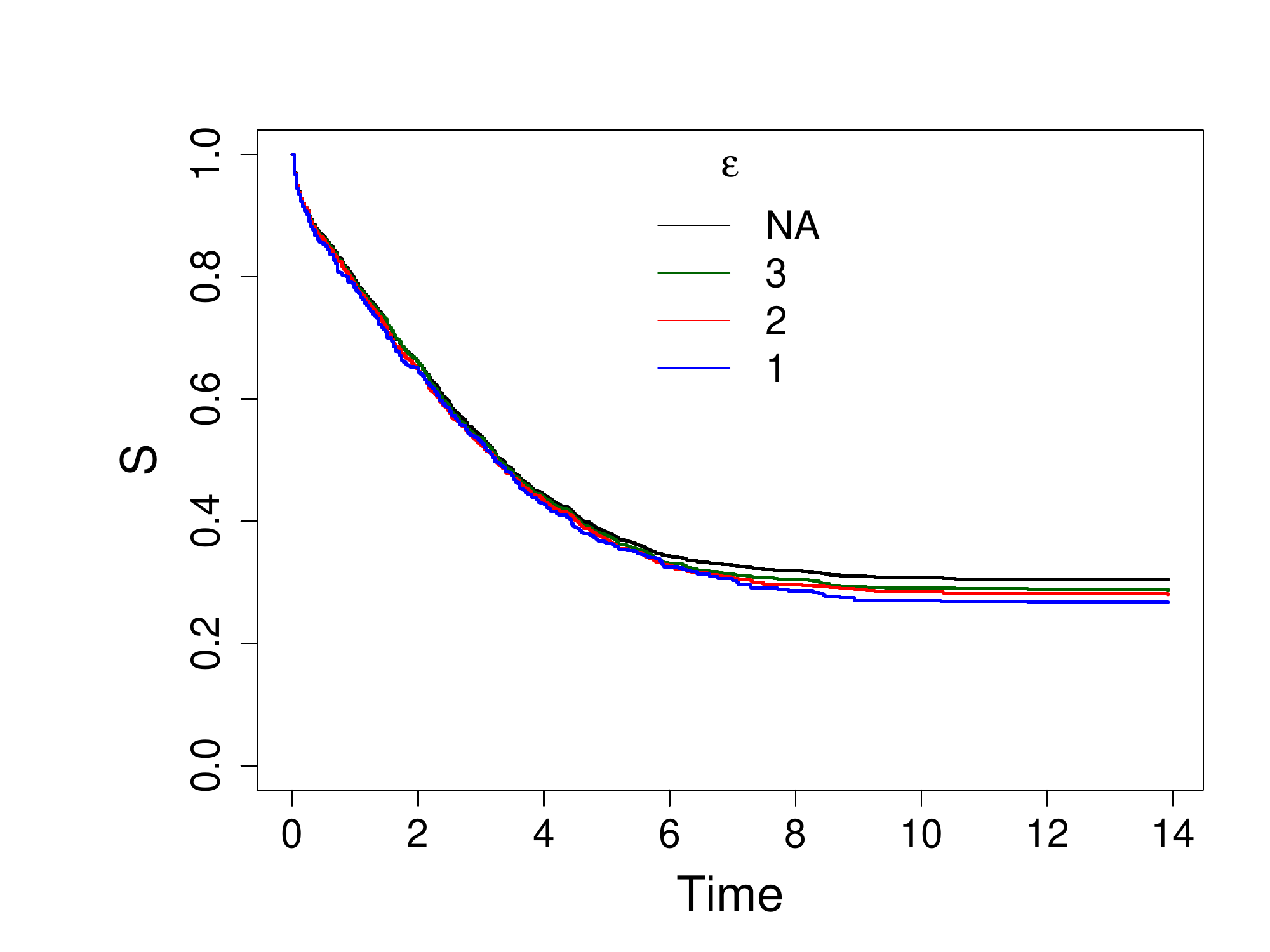}
  \caption{Mgus}
  \label{fig:sub2}
\end{subfigure}
\begin{subfigure}{.3\textwidth}
  \includegraphics[width=.9\linewidth]{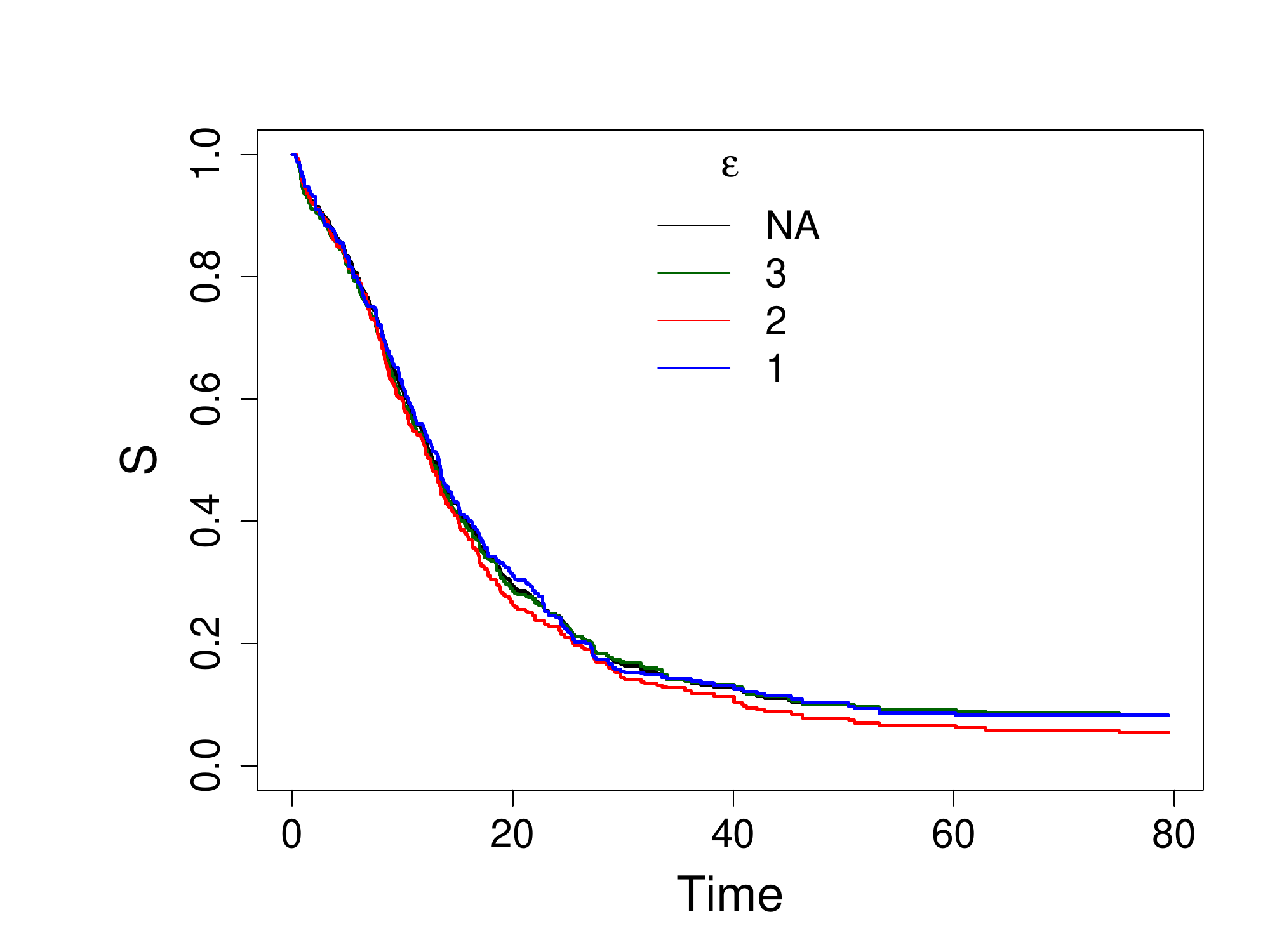}
  \caption{Myeloid}
  \label{fig:sub2}
\end{subfigure}

\begin{subfigure}{.3\textwidth}
  \includegraphics[width=.9\linewidth]{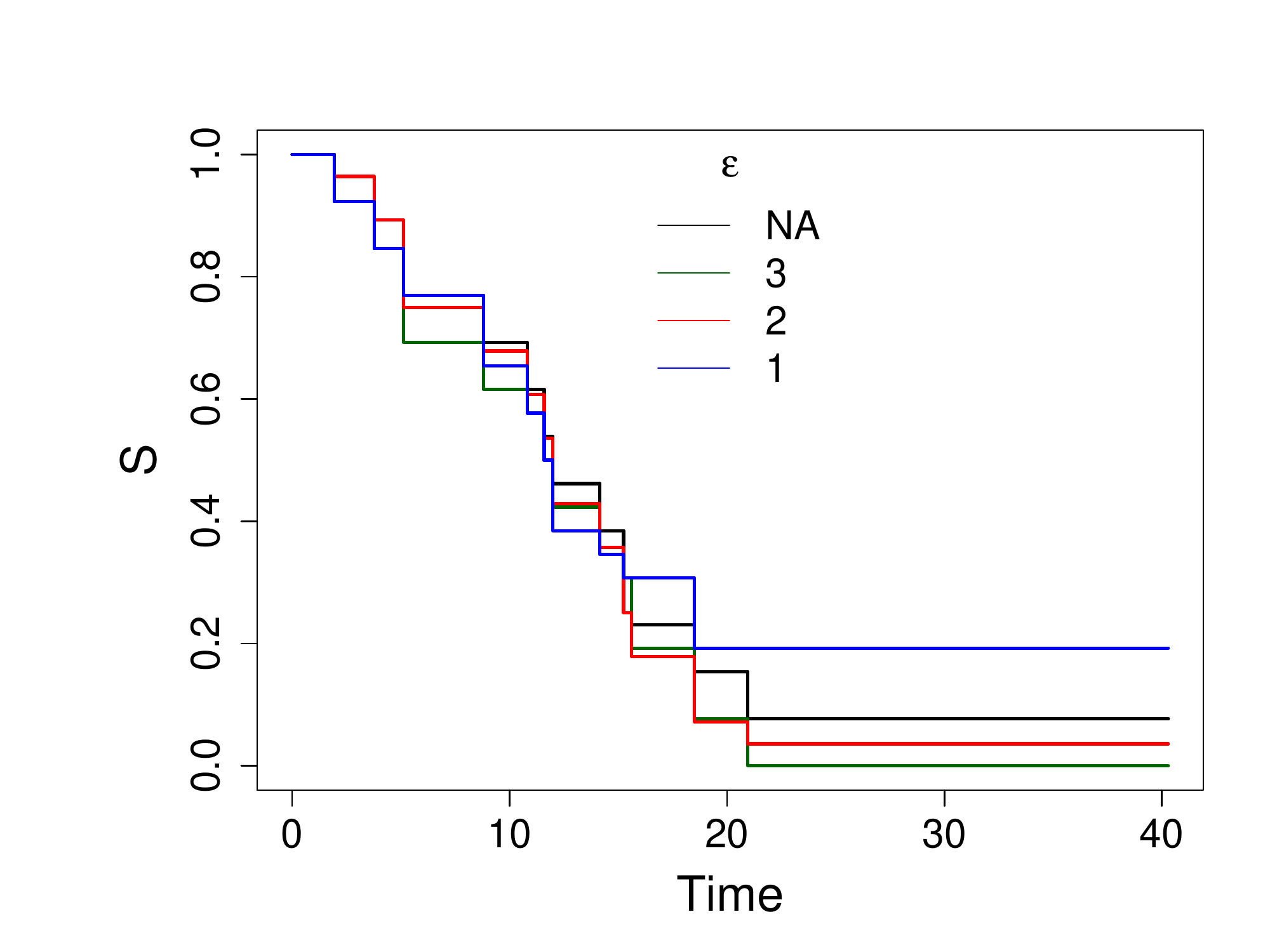}
  \caption{Ovarian}
  \label{fig:sub2}
\end{subfigure}
\begin{subfigure}{.3\textwidth}
  \includegraphics[width=.9\linewidth]{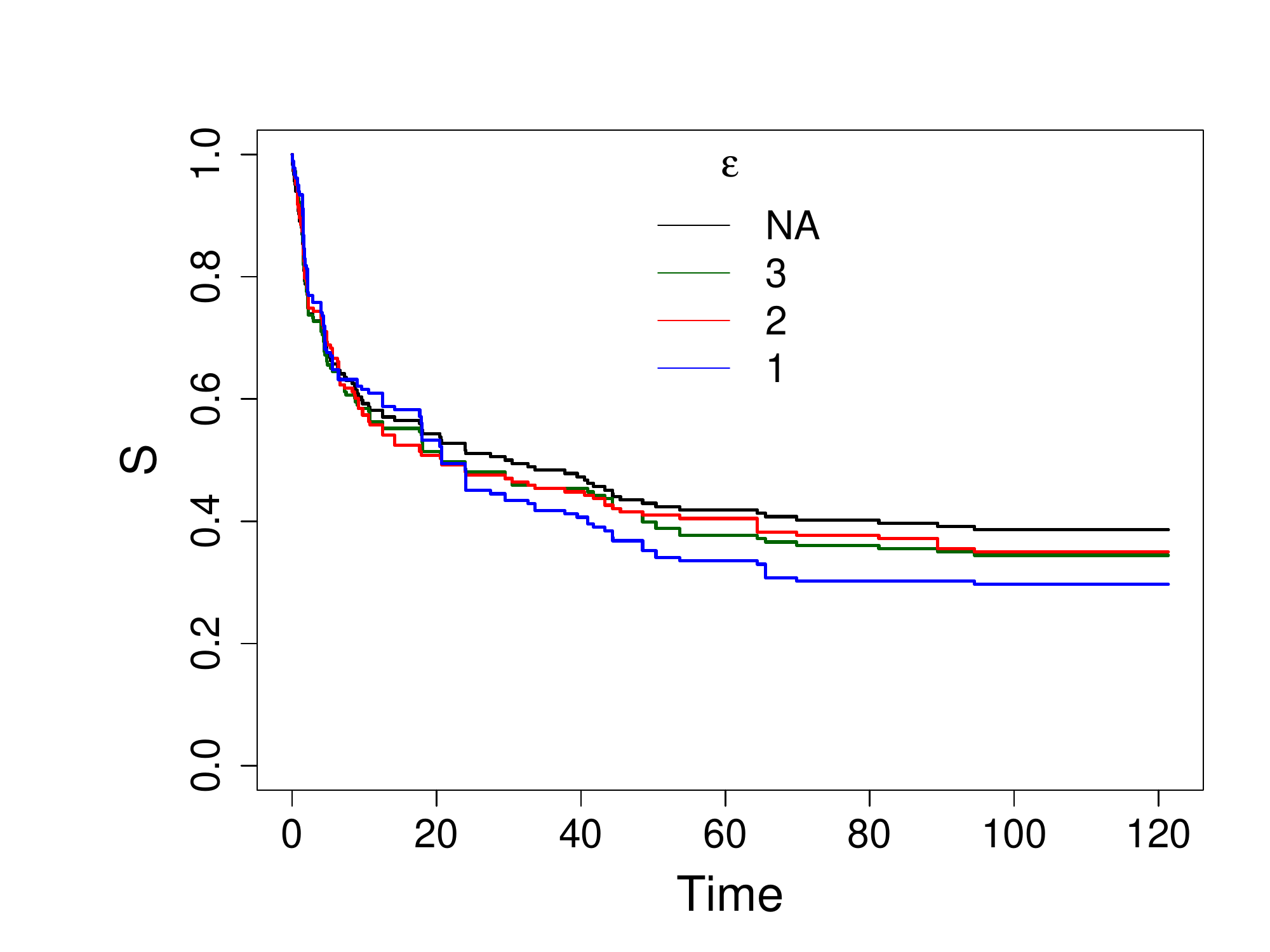}
  \caption{Stanford}
  \label{fig:sub2}
\end{subfigure}
\begin{subfigure}{.3\textwidth}
  \includegraphics[width=.9\linewidth]{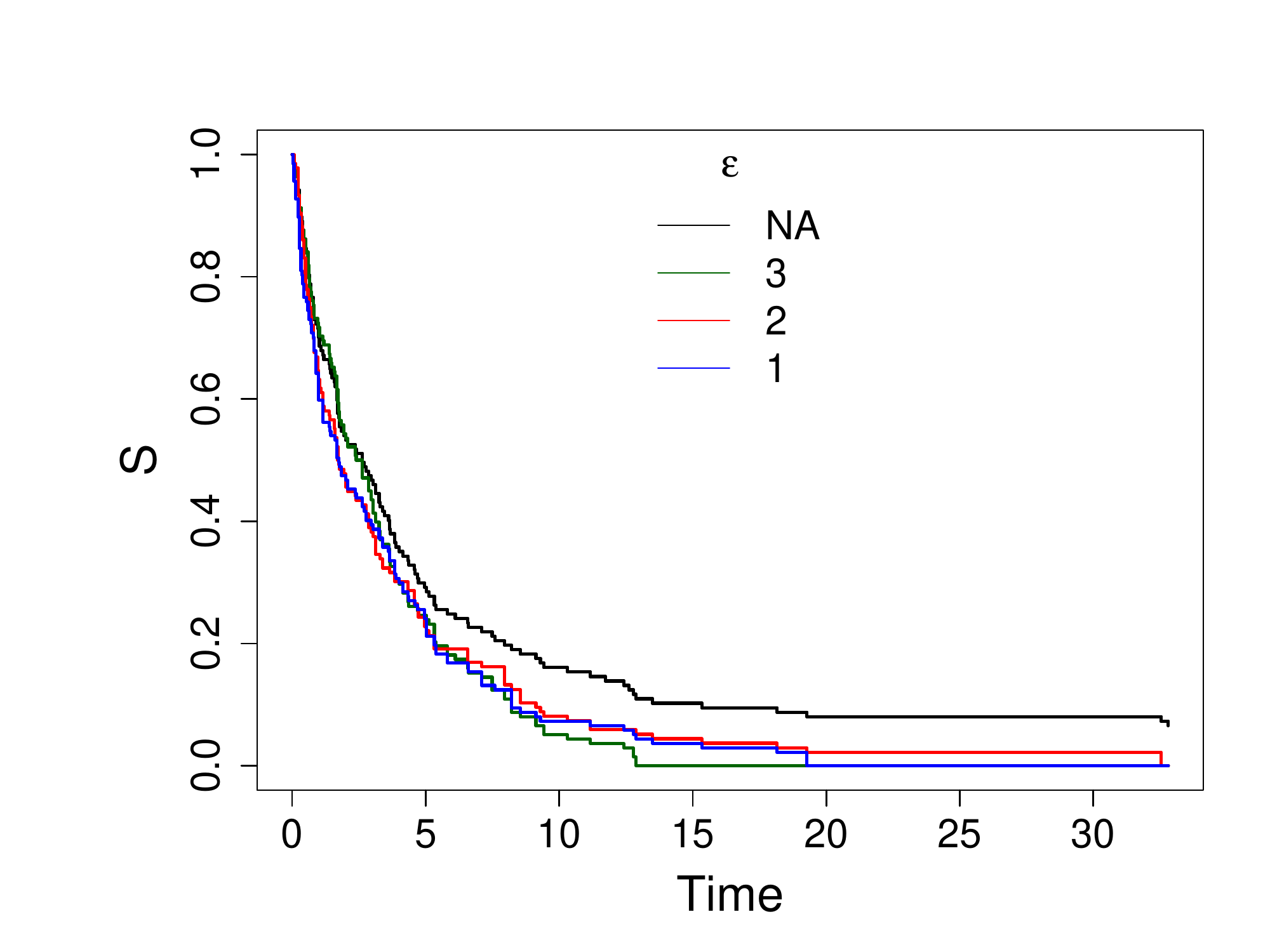}
  \caption{Veteran}
  \label{fig:sub2}
\end{subfigure}
\caption{Differentially private estimation of the survival function: Followup time is on the X-axis and the probability of survival is on the Y-axis. The black line is the original function estimate, the green line is the differentially private estimate with $\epsilon=3$, the orange line is the differentially private estimate with $\epsilon=2$, and the blue line is the differentially private estimate with $\epsilon=1$. We observe that our method provides good utility while protecting an individual's privacy. Small sample sized datasets fare worse compared to larger datasets.}
\label{fig:main}
\end{figure*}

\subsection{Setup and Comparison}
To ensure thorough evaluation of our proposed method, we use varying settings for the privacy budget $\epsilon$ ($\epsilon \in [3,2,1]$). Being a ``non-trainable" model, there are no train/test splits and results are reported on the complete dataset as an average of ten runs. All experiments are performed in R \cite{Rcore} with the source code and the datasets made publicly available on GitHub and as an R package\footnote{Link removed to respect double blind review process.}.

As there is no current method for producing differentially private estimates of the survival function. We compare our approach to the original, gold-standard of the ``non-private" estimation. This provides us with a comparison to the upper bound (we cannot get better than the non-noisy version). Good utility in comparison with the original non-perturbed version provides credibility to our claim of high utility and will encourage practitioners to adopt our method for practical use.

\subsection{Main Results}
Now we present the outcome of our evaluation of differentially private KM estimation on nine real-life datasets. We start with the estimation of the differentially private survival function and then move on to the evaluation of the extensions (confidence intervals, test statistic, etc.).

\begin{table*}[t]
\caption{Median Survival with associated confidence intervals. $\epsilon$ is the privacy budget for our method and ``No Privacy" are the results from the non-noisy model. Our method provides ``close" estimates to the original non-noisy values.}
\label{tab:median_surv}
\centering
\begin{tabular}{llllll}
\cline{1-5} 
     & \multicolumn{3}{c}{Median Survival(95\% CI)} &  \\ \cline{1-5} 
Dataset & $\epsilon=3$    & $\epsilon=2$ & $\epsilon=1$   & \begin{tabular}{@{}c@{}}No \\ Privacy\end{tabular}   &   \\ \cline{1-5}
Cancer    & 11.5 (9.5,\ 14.1)       & 11.4 (9.4,\ 12.2)          & 11.1 (9.4,\ 11.9)   &  11.5 (9.5,\ 14.1)                       \\
Gehan    & 12.0 (7.0,\ 16.0)        & 11.0 (7.0,\ 18.0)         & 11.0 (7.0,\ 22.0)  & 12.0 (8.0,\ 14.1)                      \\
Kidney    & 3.9 (1.9,\ 6.6)           & 4.9 (3.9,\ 8.0)          & 4.8 (3.5,\ 8.4)    & 4.3 (1.4,\ 6.2)                        \\
Leukemia    & 1.0 (0.3,\ 1.5)           & 0.8 (0.2,\ 1.0)          & 0.7 (0.1,\ 1.1)    & 0.9 (0.4,\ 1.4)                     \\
Mgus    & 3.3 (3.1,\ 3.5)           & 3.3 (3.0,\ 3.5)          & 3.2 (3.0,\ 3.5)    & 3.3 (3.1,\ 3.6)                       \\
Myeloid    & 12.6 (11.9,\ 13.5)        & 12.5 (11.7,\ 13.4)       & 13.4 (12.2,\ 13.8)   & 12.7 (11.9,\ 13.8)                           \\
Ovarian    & 12.0 (8.9,\ 15.6)        & 11.8 (8.8,\ 15.2)          & 11.6 (5.1,\ 15.2)   & 11.9 (8.8,\ 15.2)                      \\
Stanford    & 26.3 (15.8,\ 49.3)         & 24.7 (12.8,\ 47.3)         & 21.4 (18.7,\ 33.6) & 30.5(12.5,\ 48.5)                     \\
Veteran    & 2.6 (1.7,\ 3.0)           & 1.7 (1.2,\ 2.8)          & 1.7 (1.1,\ 2.7)    & 2.6 (1.7,\ 3.4)                        
\end{tabular}
\end{table*}

\subsubsection{Estimating Survival Function}
For the differentially private estimation of the survival function (our primary goal), Figure \ref{fig:main} shows the results. We can see that our privacy-preserving estimation (green line) faithfully estimates the survival function (black line), with little to no utility loss. As expected, estimation deteriorates with decreased privacy budget ($\epsilon \in [2,1]$, orange and blue lines respectively). This is intuitive as when the privacy budget decreases, the noise scale required to preserve differential privacy increases, leading to \emph{noisier} estimates.

An observation worth making is that as the dataset size gets smaller (such as ovarian, Leukemia, etc.), the utility of our differentially private estimation gets worse. Which is because from the differential privacy point of view, to protect an individual's privacy in a small dataset, we need to add large noise (large perturbation). Whereas for moderate to medium-sized datasets, our differentially private estimation provides good results, even for the high privacy regime. When tested for statistical differences, we found that all privacy preserving estimates (with $\epsilon \in [3,2,1]$) were not statistically-significantly different from the original, non-noisy estimate\footnote{Using the logrank test with statistical significance set at 0.05 level}.

\begin{table}[H]
\caption{The test statistic for comparing two survival distributions. $\epsilon$ is the privacy budget for our method and ``No Privacy" are the results from the non-noisy model. Our method provides good utility with strong privacy guarantees.}
\label{tab:test_stat}
\centering
\begin{tabular}{lllll}
\cline{1-5} 
 & \multicolumn{2}{c}{Test Statistic ($Z$)} \\ \cline{1-5} 
Dataset & $\epsilon=3$   & $\epsilon=2$ & $\epsilon=1$   & \begin{tabular}{@{}c@{}}No \\ Privacy\end{tabular}   \\ \cline{1-5} 
Cancer   & 11.1   & 11.8  & 11.9        & 10.3                        \\
Gehan    & 17.5  & 17.6  & 28.1        & 16.3                        \\
Kidney   & 7.4    & 8.4  & 21.9        & 6.9                        \\
Leukemia & 2.9 & 2.4   & 2.5         & 3.4                        \\
Mgus     & 7.2  & 6.9   & 5.9         & 9.7                        \\
Myeloid  & 9.2     & 9.1  & 10.7        & 9.6                        \\
Ovarian  & 0.8    & 1.2   & 2.4         & 1.1                        \\
Stanford & 6.2    & 6.9   & 8.0         & 6.6                        \\
Veteran  & 0.2    & 0.3   & 1.1         & 0.02                       
\end{tabular}
\end{table}

\subsubsection{Median Survival and Associated Confidence Intervals}
An important estimate often reported with survival function is the median survival time and its associated confidence intervals. Median survival time is defined as the time point when the survival function attains the value of $0.5$, confidence intervals for the survival function at that time point serve as the confidence intervals of the median survival. Table \ref{tab:median_surv} shows the results. For ``Median Survival (95\% CI)", we see that our method estimates the median with high precision, even for the high privacy regime. For the performance of our method, we see a similar trend as we saw with results in Figure \ref{fig:main}, where our precision increases with increasing dataset size, an acceptable trade-off for individual-level privacy protection.

\subsubsection{Test Statistic}
For the test statistic (obtained from comparing the survival distribution of different groups in the dataset, group details provided in the Section \ref{sec:datasets}), in Table \ref{tab:test_stat}, we observe that our differentially private estimation performs at par with the original ``non-noisy" estimation, even for the high privacy regime ($\epsilon \in [2,1]$). The test statistic ($Z$) follows the $\chi^2$ distribution with one degree of freedom. Using it to derive the p-values, we observe that none of the differentially private estimates change statistical significance threshold (at 0.05 level). That is, none of the differentially private estimates make the ``non-noisy" statistically significant results non-significant or vice-versa.

\subsection{Competing Risk Cumulative Incidence}\label{sec:cr}
For empirical evaluation in a competing risk scenario, we use two datasets that have more than one type of event. First is from a clinical trial for primary biliary cirrhosis (PBC) of the liver \cite{therneau2013modeling}. With the event variable being receipt of a liver transplant, censor, or death; our event of interest is the transplant, and death here is a competing event. The second dataset has the data on the subjects on a liver transplant waiting list from 1990-1999, and their disposition: received a transplant (event of interest), died while waiting (competing risk), or censored \cite{kim2006deaths}.

\begin{figure*}[h]
\centering
\begin{subfigure}{.3\textwidth}
  \includegraphics[width=.9\linewidth]{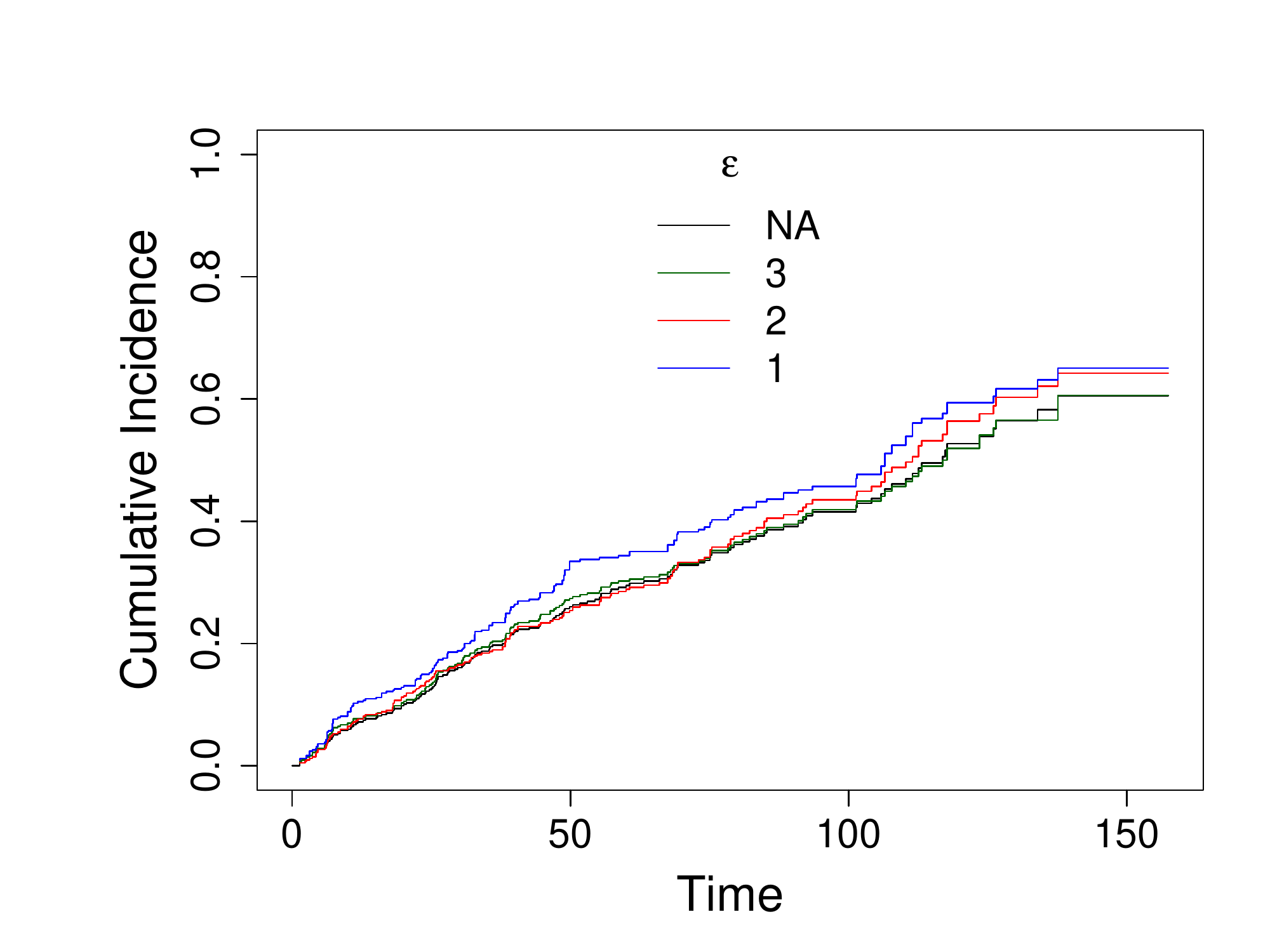}
  \caption{PBC}
  \label{fig:sub1}
\end{subfigure}
\begin{subfigure}{.3\textwidth}
  \includegraphics[width=.9\linewidth]{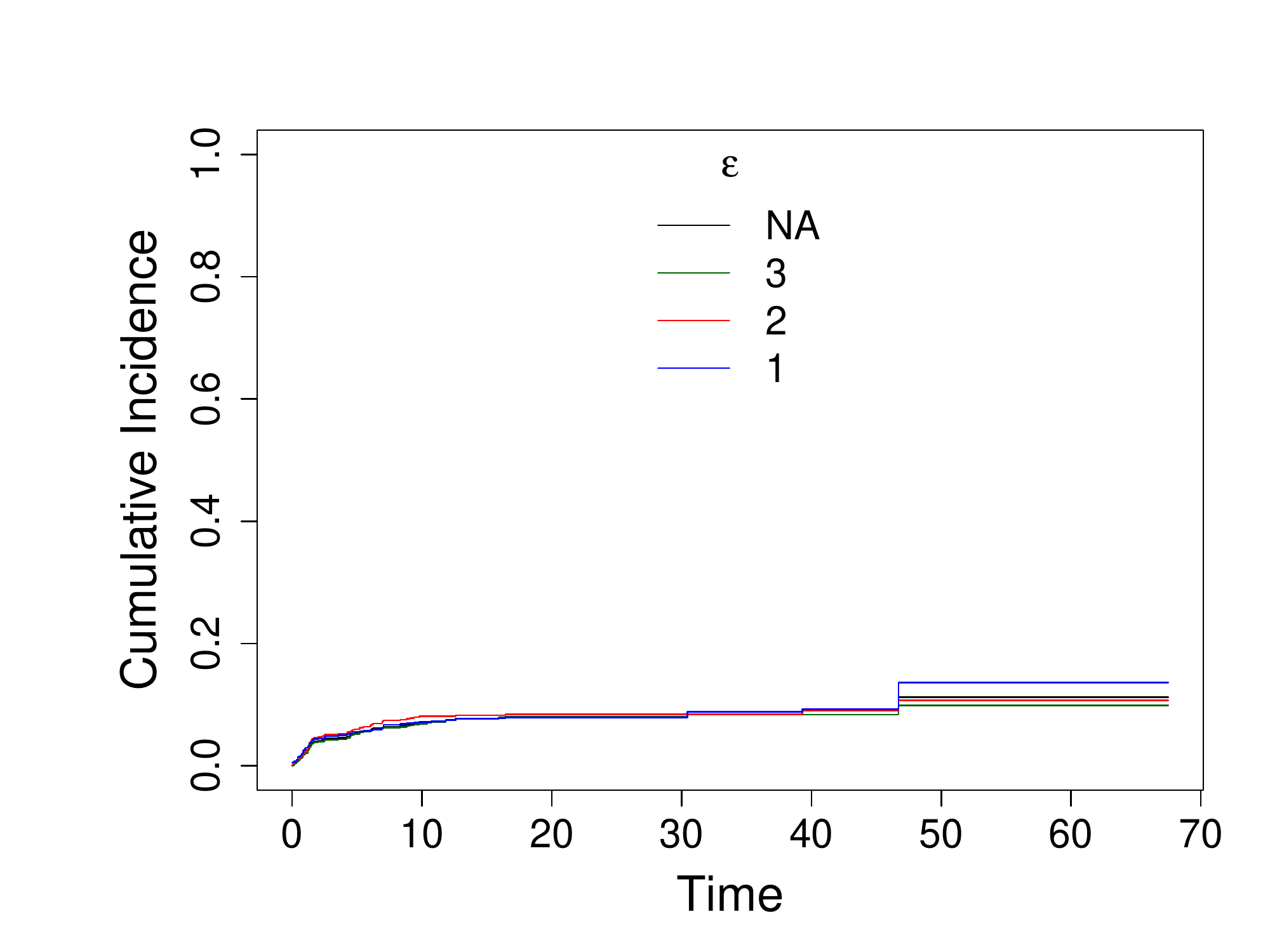}
  \caption{Transplant}
  \label{fig:sub2}
\end{subfigure}
\caption{Extending differentially private estimation to competing risk cumulative incidence (cumulative incidence is the opposite of survival function, so the plots go upward). Black is the original, unperturbed estimate. Green is with $\epsilon=3$, orange is with $\epsilon=2$, and blue is with $\epsilon=1$. We can see that our method does a good job of estimating competing risk cumulative incidence while providing strong privacy guarantees.}
\label{fig:cuminc}
\end{figure*}

\begin{figure*}[h!]
\centering
\begin{subfigure}{.3\textwidth}
  \includegraphics[width=.9\linewidth]{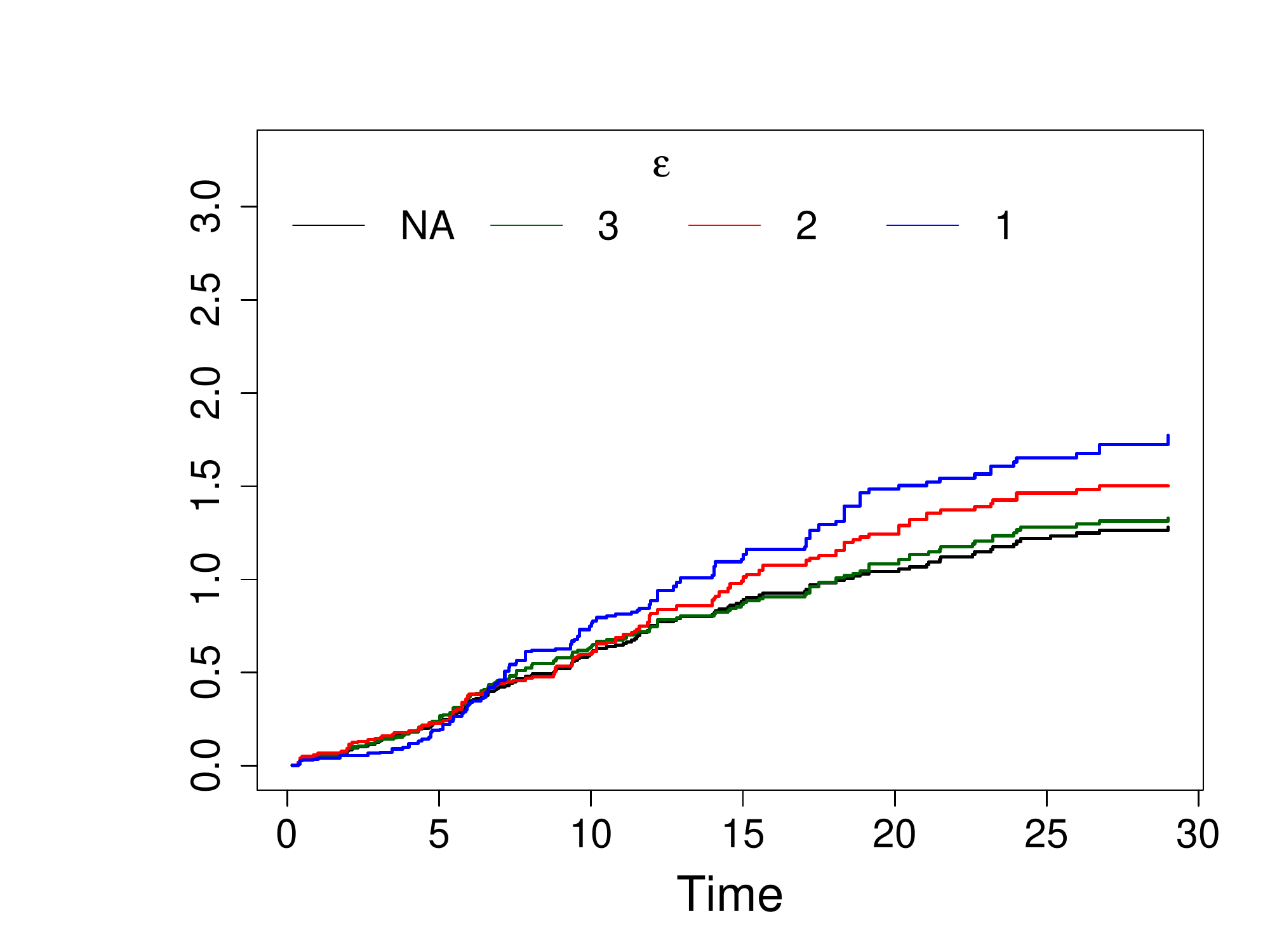}
  \caption{Cancer}
  \label{fig:sub1}
\end{subfigure}
\begin{subfigure}{.3\textwidth}
  \includegraphics[width=.9\linewidth]{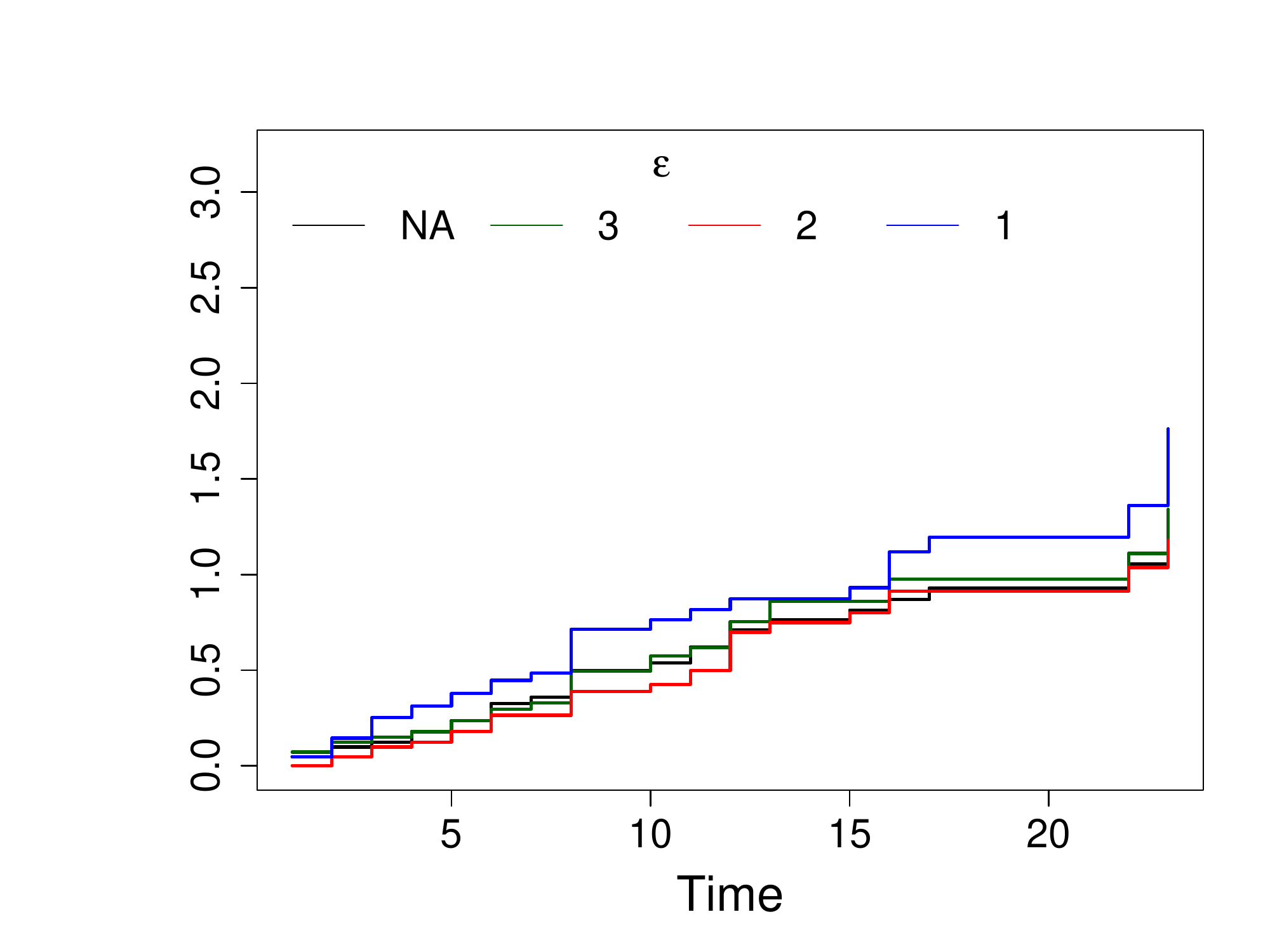}
  \caption{Gehan}
  \label{fig:sub2}
\end{subfigure}
\begin{subfigure}{.3\textwidth}
  \includegraphics[width=.9\linewidth]{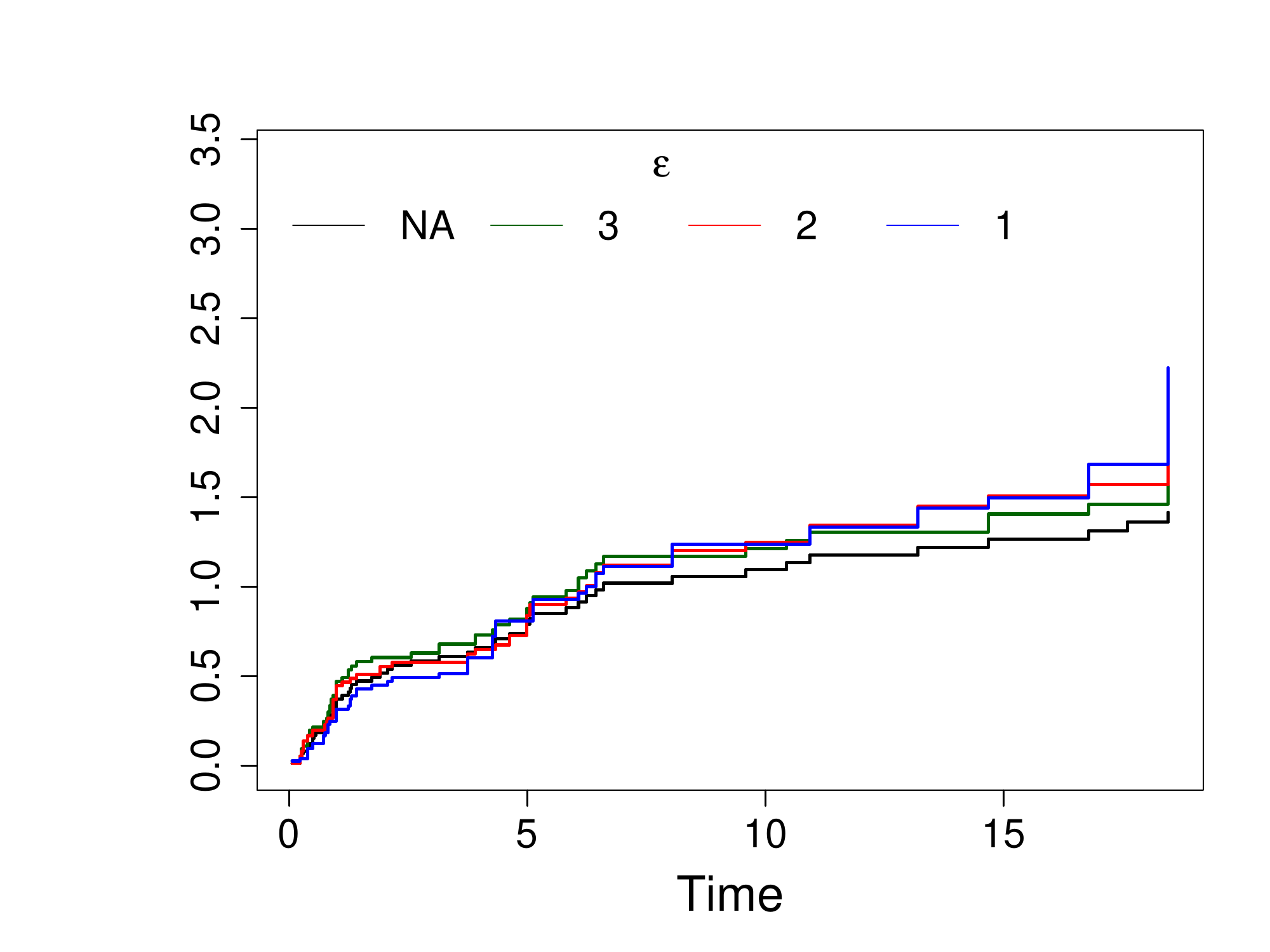}
  \caption{Kidney}
  \label{fig:sub2}
\end{subfigure}

\begin{subfigure}{.3\textwidth}
  \includegraphics[width=.9\linewidth]{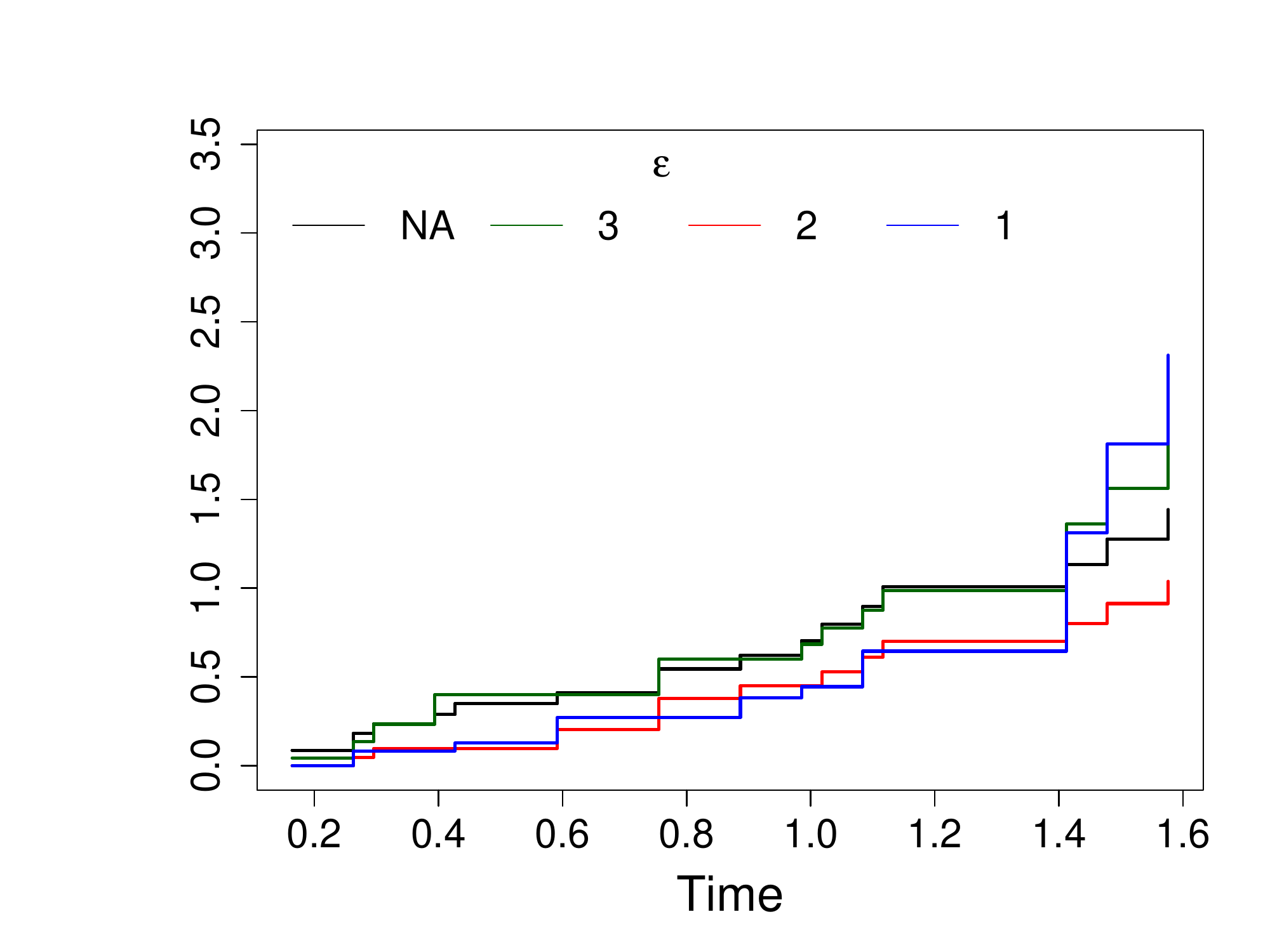}
  \caption{Leukemia}
  \label{fig:sub2}
\end{subfigure}
\begin{subfigure}{.3\textwidth}
  \includegraphics[width=.9\linewidth]{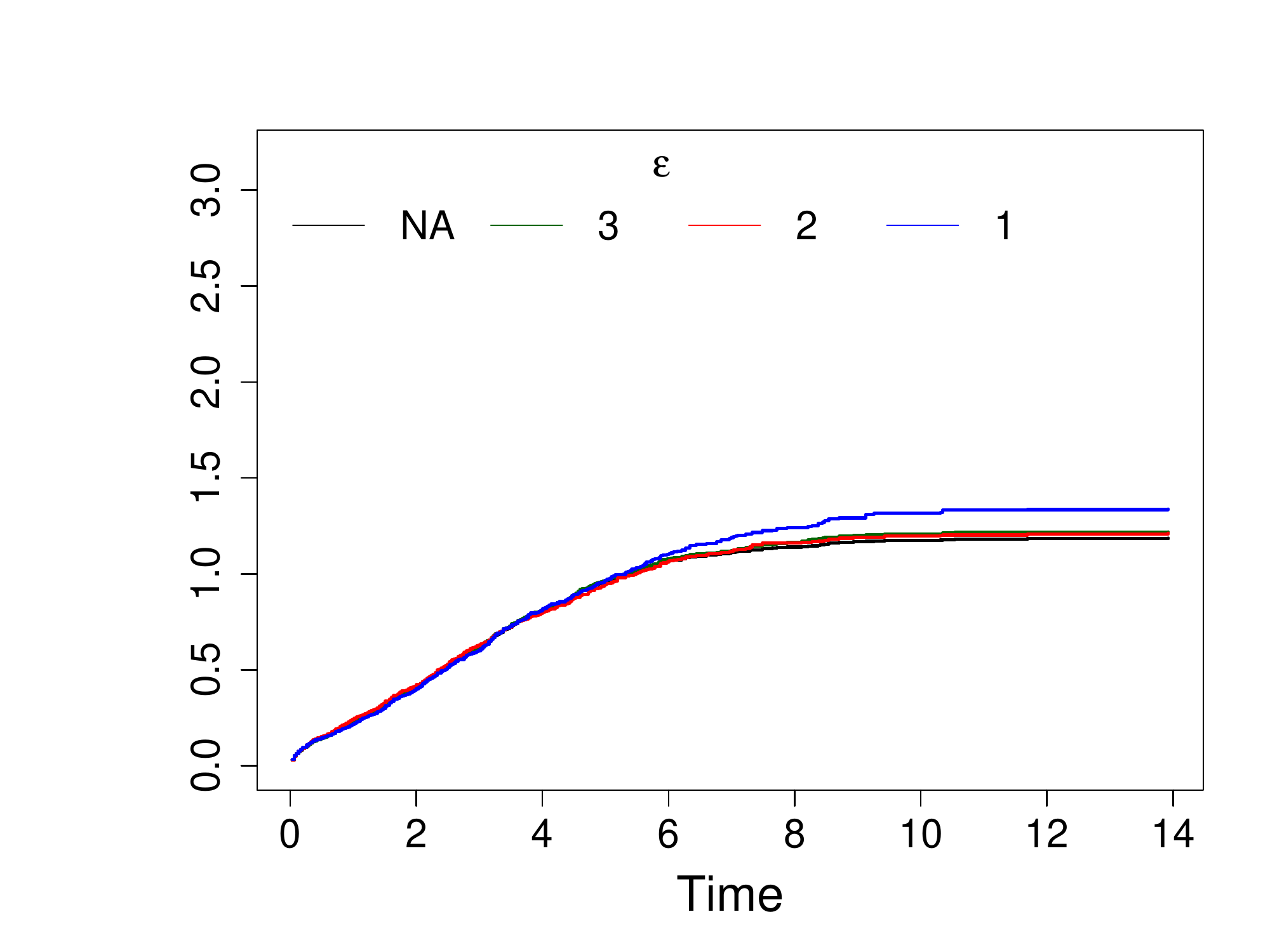}
  \caption{Mgus}
  \label{fig:sub2}
\end{subfigure}
\begin{subfigure}{.3\textwidth}
  \includegraphics[width=.9\linewidth]{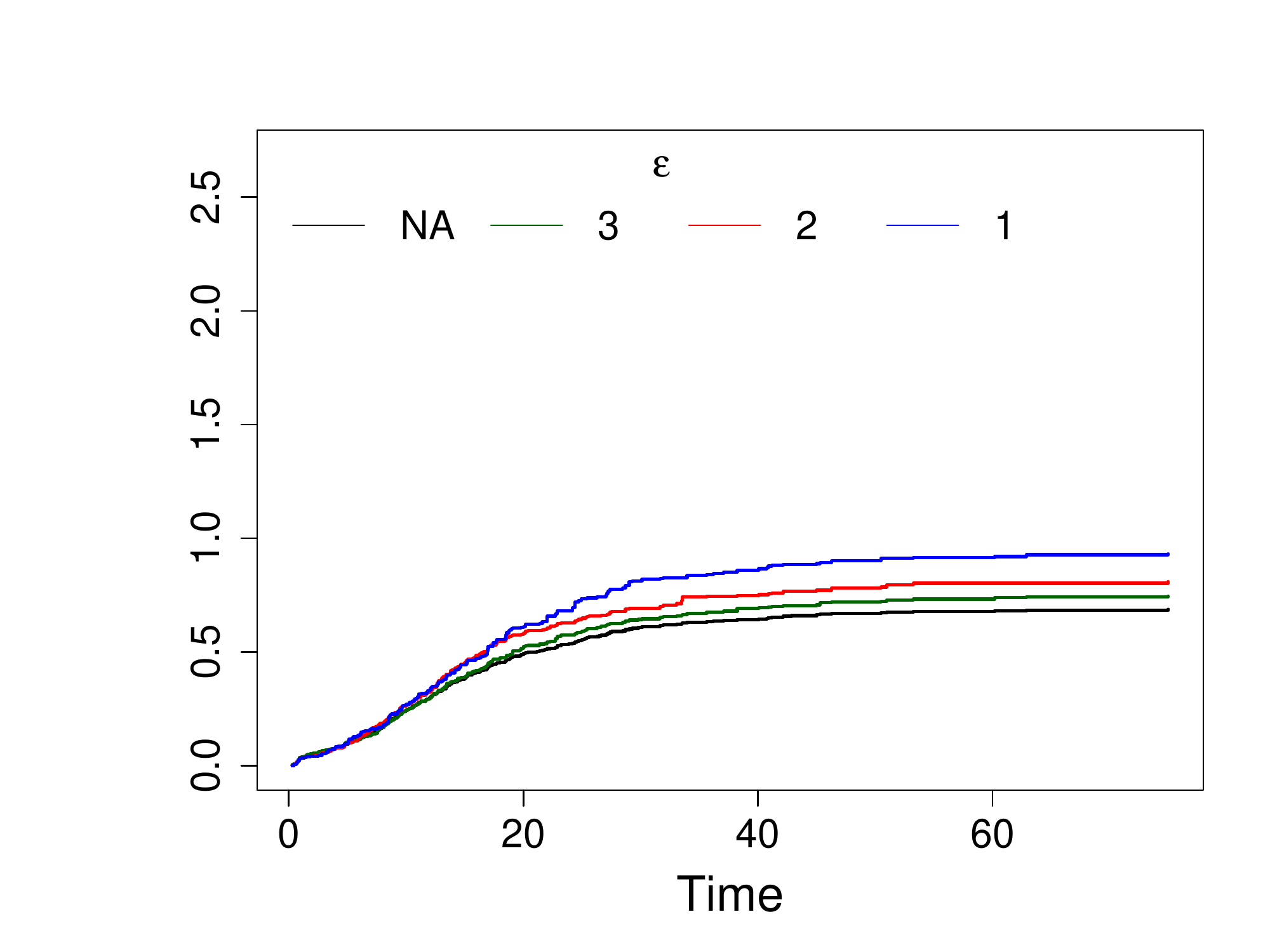}
  \caption{Myeloid}
  \label{fig:sub2}
\end{subfigure}

\begin{subfigure}{.3\textwidth}
  \includegraphics[width=.9\linewidth]{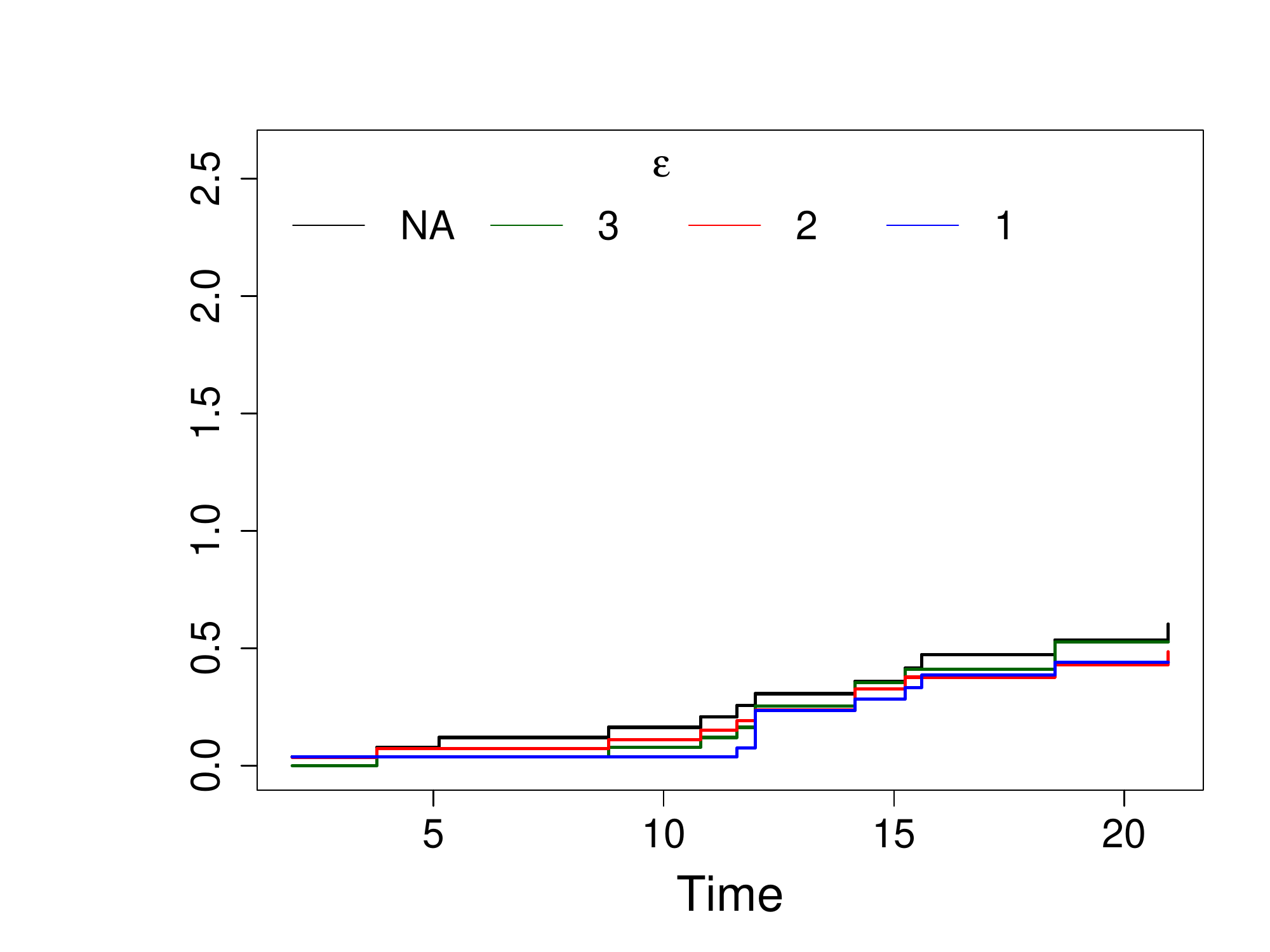}
  \caption{Ovarian}
  \label{fig:sub2}
\end{subfigure}
\begin{subfigure}{.3\textwidth}
  \includegraphics[width=.9\linewidth]{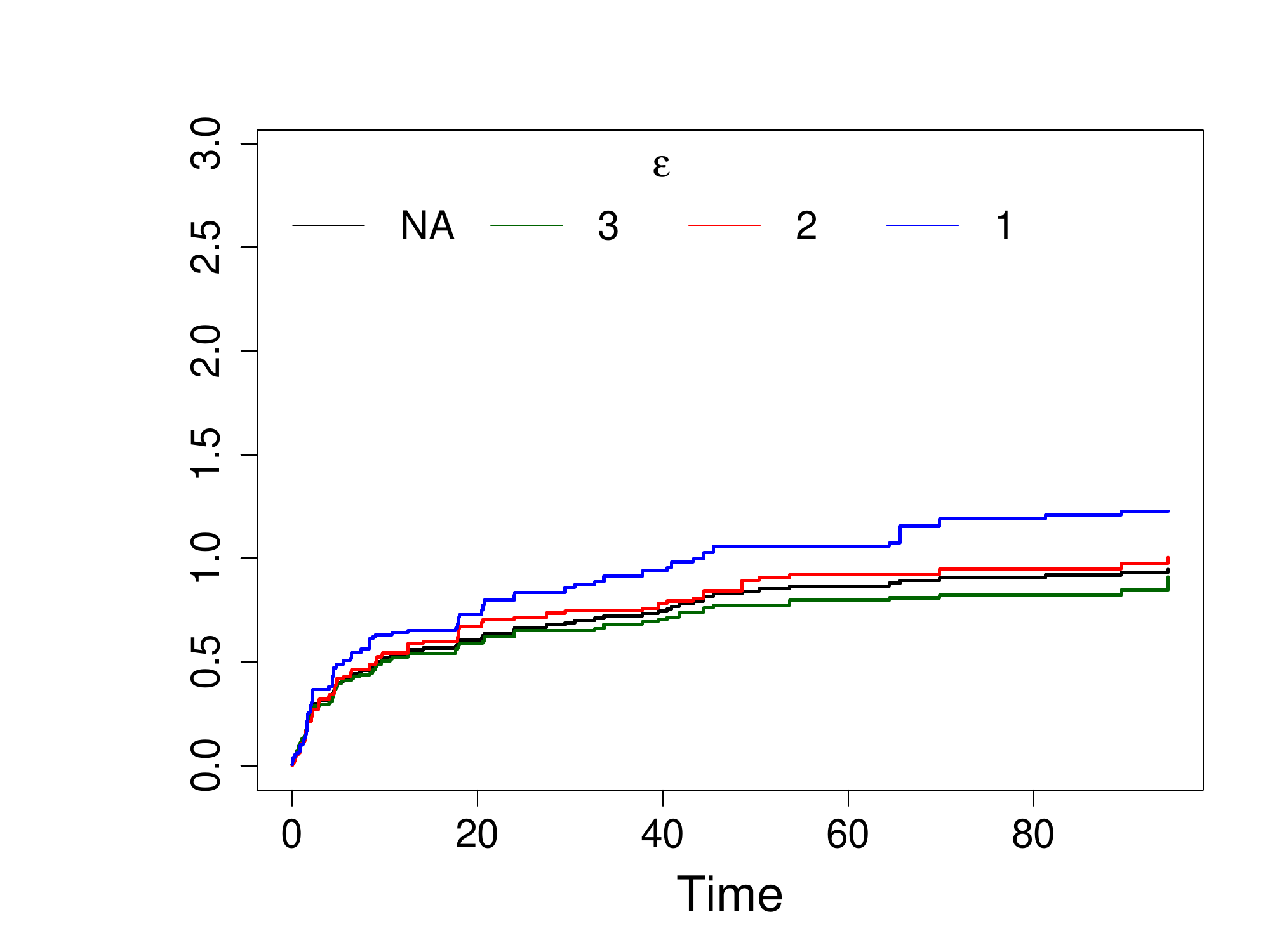}
  \caption{Stanford}
  \label{fig:sub2}
\end{subfigure}
\begin{subfigure}{.3\textwidth}
  \includegraphics[width=.9\linewidth]{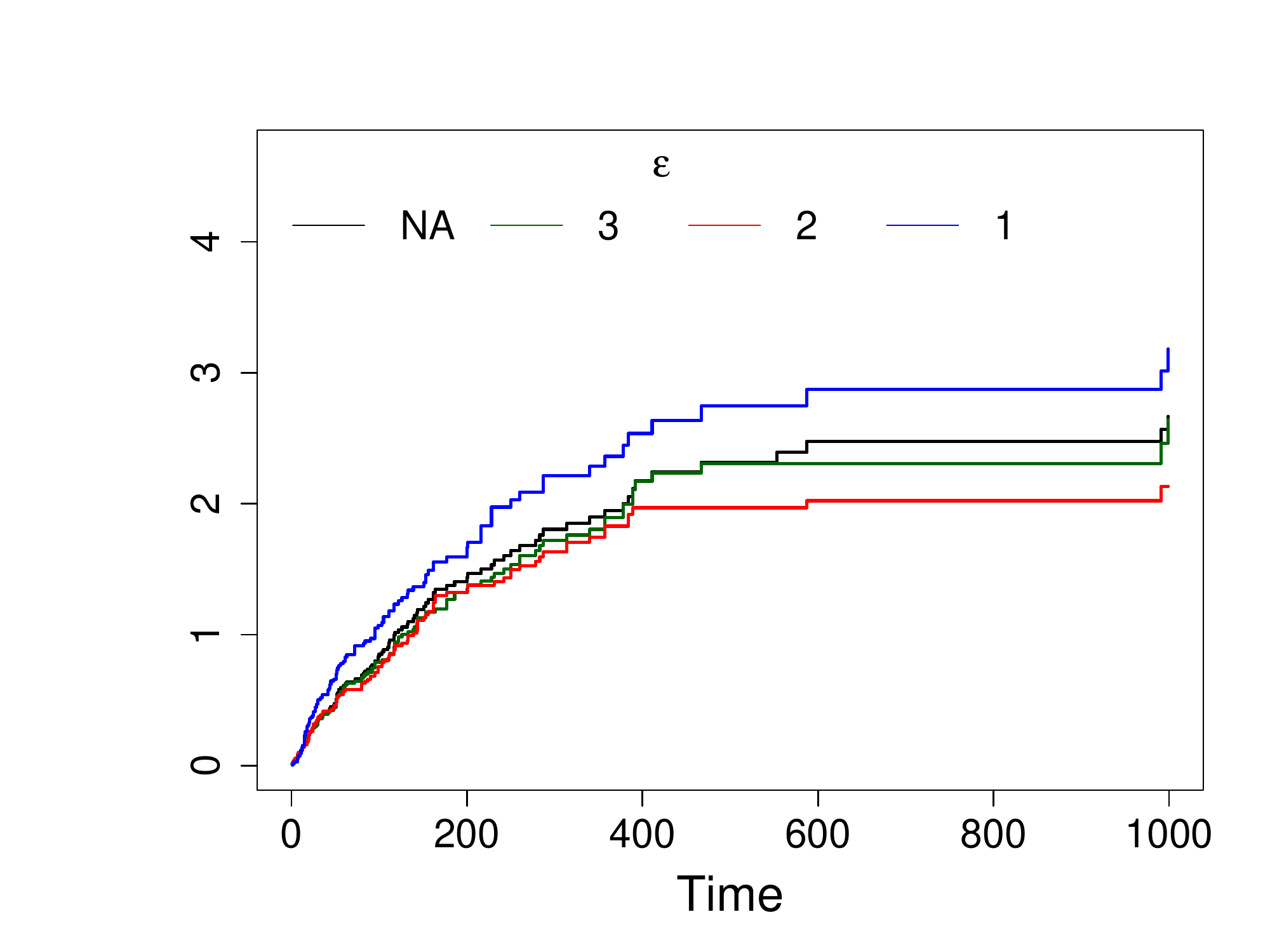}
  \caption{Veteran}
  \label{fig:sub2}
\end{subfigure}
\caption{Differentially private estimation of the Nelson-Aalen estimator, followup time is on the X-axis and the hazard estimate is on the Y-axis. The black line is the original function estimate, the green line is the differentially private estimate with $\epsilon=3$, the orange line is the differentially private estimate with $\epsilon=2$, and the blue line is the differentially private estimate with $\epsilon=1$. We observe that our differentially private version provides good utility while protecting an individual's privacy.}
\label{fig:na_main}
\end{figure*}

Figure \ref{fig:cuminc} shows the results (cumulative incidence is the opposite of survival function, so the plots go upward). We observe that our differentially private extension does an excellent job of differentially private estimation of the competing risk cumulative incidence function while providing strong privacy guarantees.

\subsection{Nelson-Aalen Estimate}
For evaluating the performance of our proposed differentially private Nelson-Aalen's estimator of the hazard function, we use the main nine datasets. Please note that similar to the competing risk cumulative incidence, being a ``risk" estimate, the value of the cumulative hazard estimate increases over time, hence it has an ``upward" curve compared to the ``downward" curve for the survival estimate.

Figure \ref{fig:na_main} shows the results for all nine datasets. Our differentially private estimate performs extremely well, similar to our main comparison, where we can see that our estimation provides good utility, even at high privacy regimes. Also, similar to our main comparison, all differentially private estimates are not statistically-significantly different from the original, non-noisy estimates.

\section{Related Work}
Much work has been done in the intersection of statistical modeling and differential privacy, including many works proposing different differentially private methods for regression modeling \cite{sheffet2017differentially,jain2012differentially,zhang2012functional,yu2014differentially,chaudhuri2011differentially}. Using the same principles, \cite{nguyen2017differentially} further developed a differentially private regression model for survival analysis. This approach is limited to the ``multivariate" regression models and cannot be used for direct differentially private estimation of the survival function. Differentially private generative models such as the differentially private generative adversarial networks \cite{xie2018differentially,zhang2018differentially,esteban2017real,triastcyn2018generating,beaulieu2017privacy,yoon2018pategan} have been recently proposed. But, as discussed in the introduction, they are not suitable for generating data for survival function estimation.

\section{Conclusion}
We have presented the first method for differentially private estimation of the survival function and we have shown that our proposed method can be easily extended to differentially private estimation of ``other" often used statistics such as the associated confidence intervals, test statistics, and to other estimates such as the competing risk cumulative incidence and the Nelson-Aalen estimate of the hazard function. With extensive empirical evaluation on eleven real-life datasets, we have shown that our proposed method provides a good privacy-utility trade-off. For our future work, we would like to investigate the impact of considering different privacy definitions and noise distributions.

\section{Acknowledgement}
Lovedeep Gondara is supported by an NSERC (Natural Sciences and Engineering Research Council of Canada) CGS D award.

\bibliographystyle{ieeetr}
\bibliography{nips}

\begin{thebibliography}{10}

\bibitem{wei2018reconstructing}
Y.~Wei and P.~Royston, ``Reconstructing time-to-event data from published
  kaplan--meier curves,'' {\em The Stata Journal}, vol.~17, no.~4,
  pp.~786--802, 2018.

\bibitem{fredrikson2014privacy}
M.~Fredrikson, E.~Lantz, S.~Jha, S.~Lin, D.~Page, and T.~Ristenpart, ``Privacy
  in pharmacogenetics: An end-to-end case study of personalized warfarin
  dosing,'' in {\em 23rd $\{$USENIX$\}$ Security Symposium ($\{$USENIX$\}$
  Security 14)}, pp.~17--32, 2014.

\bibitem{kaplan1958nonparametric}
E.~L. Kaplan and P.~Meier, ``Nonparametric estimation from incomplete
  observations,'' {\em Journal of the American statistical association},
  vol.~53, no.~282, pp.~457--481, 1958.

\bibitem{nguyen2017differentially}
T.~T. Nguy{\^e}n and S.~C. Hui, ``Differentially private regression for
  discrete-time survival analysis,'' in {\em Proceedings of the 2017 ACM on
  Conference on Information and Knowledge Management}, pp.~1199--1208, ACM,
  2017.

\bibitem{xie2018differentially}
L.~Xie, K.~Lin, S.~Wang, F.~Wang, and J.~Zhou, ``Differentially private
  generative adversarial network,'' {\em arXiv preprint arXiv:1802.06739},
  2018.

\bibitem{zhang2018differentially}
X.~Zhang, S.~Ji, and T.~Wang, ``Differentially private releasing via deep
  generative model,'' {\em arXiv preprint arXiv:1801.01594}, 2018.

\bibitem{triastcyn2018generating}
A.~Triastcyn and B.~Faltings, ``Generating differentially private datasets
  using gans,'' {\em arXiv preprint arXiv:1803.03148}, 2018.

\bibitem{beaulieu2017privacy}
B.~K. Beaulieu-Jones, Z.~S. Wu, C.~Williams, and C.~S. Greene,
  ``Privacy-preserving generative deep neural networks support clinical data
  sharing,'' {\em BioRxiv}, p.~159756, 2017.

\bibitem{esteban2017real}
C.~Esteban, S.~L. Hyland, and G.~R{\"a}tsch, ``Real-valued (medical) time
  series generation with recurrent conditional gans,'' {\em arXiv preprint
  arXiv:1706.02633}, 2017.

\bibitem{yoon2018pategan}
J.~Yoon, J.~Jordon, and M.~van~der Schaar, ``{PATE}-{GAN}: Generating synthetic
  data with differential privacy guarantees,'' in {\em International Conference
  on Learning Representations}, 2019.

\bibitem{nelson1972theory}
W.~Nelson, ``Theory and applications of hazard plotting for censored failure
  data,'' {\em Technometrics}, vol.~14, no.~4, pp.~945--966, 1972.

\bibitem{nelson1969hazard}
W.~Nelson, ``Hazard plotting for incomplete failure data,'' {\em Journal of
  Quality Technology}, vol.~1, no.~1, pp.~27--52, 1969.

\bibitem{aalen1978nonparametric}
O.~Aalen, ``Nonparametric inference for a family of counting processes,'' {\em
  The Annals of Statistics}, pp.~701--726, 1978.

\bibitem{Rcore}
{R Core Team}, {\em R: A Language and Environment for Statistical Computing}.
\newblock R Foundation for Statistical Computing, Vienna, Austria, 2018.

\bibitem{dinur2003revealing}
I.~Dinur and K.~Nissim, ``Revealing information while preserving privacy,'' in
  {\em Proceedings of the twenty-second ACM SIGMOD-SIGACT-SIGART symposium on
  Principles of database systems}, pp.~202--210, ACM, 2003.

\bibitem{Dwork:2006:CNS:2180286.2180305}
C.~Dwork, F.~McSherry, K.~Nissim, and A.~Smith, ``Calibrating noise to
  sensitivity in private data analysis,'' in {\em Proceedings of the Third
  Conference on Theory of Cryptography}, TCC'06, (Berlin, Heidelberg),
  pp.~265--284, Springer-Verlag, 2006.

\bibitem{wang2016using}
Y.~Wang, X.~Wu, and D.~Hu, ``Using randomized response for differential privacy
  preserving data collection.,'' in {\em EDBT/ICDT Workshops}, vol.~1558, 2016.

\bibitem{flaxman2019empirical}
A.~D. Flaxman, ``Empirical quantification of privacy loss with examples
  relevant to the 2020 us census,'' 2019.

\bibitem{Dwork:2014:AFD:2693052.2693053}
C.~Dwork and A.~Roth, ``The algorithmic foundations of differential privacy,''
  {\em Found. Trends Theor. Comput. Sci.}, vol.~9, pp.~211--407, Aug. 2014.

\bibitem{greenwood1926report}
M.~Greenwood {\em et~al.}, ``A report on the natural duration of cancer.,''
  {\em A Report on the Natural Duration of Cancer.}, no.~33, 1926.

\bibitem{mantel1966evaluation}
N.~Mantel, ``Evaluation of survival data and two new rank order statistics
  arising in its consideration,'' {\em Cancer Chemother Rep}, vol.~50,
  pp.~163--170, 1966.

\bibitem{loprinzi1994prospective}
C.~L. Loprinzi, J.~A. Laurie, H.~S. Wieand, J.~E. Krook, P.~J. Novotny, J.~W.
  Kugler, J.~Bartel, M.~Law, M.~Bateman, and N.~E. Klatt, ``Prospective
  evaluation of prognostic variables from patient-completed questionnaires.
  north central cancer treatment group.,'' {\em Journal of Clinical Oncology},
  vol.~12, no.~3, pp.~601--607, 1994.

\bibitem{cox2018analysis}
D.~R. Cox, {\em Analysis of survival data}.
\newblock Routledge, 2018.

\bibitem{mcgilchrist1991regression}
C.~McGilchrist and C.~Aisbett, ``Regression with frailty in survival
  analysis.,'' {\em Biometrics}, vol.~47, no.~2, pp.~461--466, 1991.

\bibitem{miller2011survival}
R.~G. Miller~Jr, {\em Survival analysis}, vol.~66.
\newblock John Wiley \& Sons, 2011.

\bibitem{kyle1993benign}
R.~A. Kyle, ``“benign” monoclonal gammopathy—after 20 to 35 years of
  follow-up,'' in {\em Mayo Clinic Proceedings}, vol.~68, pp.~26--36, Elsevier,
  1993.

\bibitem{edmonson1979different}
J.~H. Edmonson, T.~R. Fleming, D.~Decker, G.~Malkasian, E.~Jorgensen,
  J.~Jefferies, M.~Webb, and L.~Kvols, ``Different chemotherapeutic
  sensitivities and host factors affecting prognosis in advanced ovarian
  carcinoma versus minimal residual disease.,'' {\em Cancer treatment reports},
  vol.~63, no.~2, pp.~241--247, 1979.

\bibitem{escobar1992assessing}
L.~A. Escobar and W.~Q. Meeker~Jr, ``Assessing influence in regression analysis
  with censored data,'' {\em Biometrics}, pp.~507--528, 1992.

\bibitem{kalbfleisch2011statistical}
J.~D. Kalbfleisch and R.~L. Prentice, {\em The statistical analysis of failure
  time data}, vol.~360.
\newblock John Wiley \& Sons, 2011.

\bibitem{therneau2013modeling}
T.~M. Therneau and P.~M. Grambsch, {\em Modeling survival data: extending the
  Cox model}.
\newblock Springer Science \& Business Media, 2013.

\bibitem{kim2006deaths}
W.~R. Kim, T.~M. Therneau, J.~T. Benson, W.~K. Kremers, C.~B. Rosen, G.~J.
  Gores, and E.~R. Dickson, ``Deaths on the liver transplant waiting list: an
  analysis of competing risks,'' {\em Hepatology}, vol.~43, no.~2,
  pp.~345--351, 2006.

\bibitem{sheffet2017differentially}
O.~Sheffet, ``Differentially private ordinary least squares,'' in {\em
  Proceedings of the 34th International Conference on Machine Learning-Volume
  70}, pp.~3105--3114, JMLR. org, 2017.

\bibitem{jain2012differentially}
P.~Jain, P.~Kothari, and A.~Thakurta, ``Differentially private online
  learning,'' in {\em Conference on Learning Theory}, pp.~24--1, 2012.

\bibitem{zhang2012functional}
J.~Zhang, Z.~Zhang, X.~Xiao, Y.~Yang, and M.~Winslett, ``Functional mechanism:
  regression analysis under differential privacy,'' {\em Proceedings of the
  VLDB Endowment}, vol.~5, no.~11, pp.~1364--1375, 2012.

\bibitem{yu2014differentially}
F.~Yu, M.~Rybar, C.~Uhler, and S.~E. Fienberg, ``Differentially-private
  logistic regression for detecting multiple-snp association in gwas
  databases,'' in {\em International Conference on Privacy in Statistical
  Databases}, pp.~170--184, Springer, 2014.

\bibitem{chaudhuri2011differentially}
K.~Chaudhuri, C.~Monteleoni, and A.~D. Sarwate, ``Differentially private
  empirical risk minimization,'' {\em Journal of Machine Learning Research},
  vol.~12, no.~Mar, pp.~1069--1109, 2011.

\end{thebibliography}
\end{document}